\documentclass{article}
\usepackage{amsmath,amssymb,amsthm}
\usepackage{authblk}
\usepackage{graphicx}
\usepackage{natbib}
\usepackage{url}
\usepackage{bm}
\usepackage{color}
\usepackage{booktabs}
\usepackage{siunitx}
\def\E{\mathrm{E}}

\def\N{\mathbb{N}}
\def\rv{\mathrm{v}}

\DeclareMathOperator{\diag}{diag}
\DeclareMathOperator{\rank}{rank}

\DeclareMathOperator*{\aff}{aff}
\usepackage{algorithm}
\usepackage[noend]{algpseudocode}
\usepackage{float}

\usepackage[a4paper,left=3cm,right=3cm,top=3cm,bottom=3cm]{geometry}

\newcommand{\R}{\mathbb{R}}

\newcommand{\bA}{\textbf{A}}
\newcommand{\bB}{\textbf{B}}

\newcommand{\bD}{\textbf{D}}
\newcommand{\bE}{\textbf{E}}

\newcommand{\bH}{\textbf{H}}
\newcommand{\bI}{\textbf{I}}
\newcommand{\bJ}{\textbf{J}}

\newcommand{\bL}{\textbf{L}}
\newcommand{\bM}{\textbf{M}}

\newcommand{\bP}{\textbf{P}}
\newcommand{\bQ}{\textbf{Q}}
\newcommand{\bR}{\textbf{R}}
\newcommand{\bS}{\textbf{S}}
\newcommand{\bT}{\textbf{T}}
\newcommand{\bU}{\textbf{U}}
\newcommand{\bV}{\textbf{V}}
\newcommand{\bW}{\textbf{W}}
\newcommand{\bX}{\textbf{X}}

\newcommand{\bDelta}{\boldsymbol{\Delta}}

\newcommand{\bSigma}{\boldsymbol{\Sigma}}

\newcommand{\bGamma}{\boldsymbol{\Gamma}}

\newcommand{\rmMax}{\textrm{max}}

\def\acosh{\mathrm{arcosh}}
\def\acos{\mathrm{arcos}}
\newcommand{\bZero}{\mathbf{0}}

\newcommand{\Ex}{\mathbb{E}}

\newcommand{\probO}{O_{\mathbb{P}}}
\newcommand{\probo}{o_{\mathbb{P}}}

\newcommand{\ttinf}{2\rightarrow\infty}

\newcommand{\indefO}{\mathbb{O}(p,q)}
\newcommand{\indefI}{\bI_{p,q}}
\newcommand{\defO}{\mathbb{O}}

\newtheorem{theorem}{Theorem}

\newtheorem{lemma}[theorem]{Lemma}

\theoremstyle{definition}
\newtheorem{remark}{Remark}
\newtheorem{definition}[theorem]{Definition}

\theoremstyle{definition}

\newtheorem{property}[theorem]{Property}

\author[1]{Patrick Rubin-Delanchy}
\author[2]{Joshua Cape}
\author[3]{Minh Tang}
\author[4]{Carey E. Priebe}
\affil[1]{University of Bristol and Heilbronn Institute for Mathematical Research, U.K.}
\affil[2]{University of Pittsburgh, U.S.A.}
\affil[3]{North Carolina State University, U.S.A.}
\affil[4]{Johns Hopkins University, U.S.A.}
\date{}
\title{A statistical interpretation of spectral embedding: \\the generalised random dot product graph}

\begin{document}
\maketitle

\begin{abstract}
  Spectral embedding is a procedure which can be used to obtain vector representations of the nodes of a graph. This paper proposes a generalisation of the latent position network model known as the random dot product graph, to allow interpretation of those vector representations as latent position estimates. The generalisation is needed to model heterophilic connectivity (e.g., `opposites attract') and to cope with negative eigenvalues more generally. We show that, whether the adjacency or normalised Laplacian matrix is used, spectral embedding produces uniformly consistent latent position estimates with asymptotically Gaussian error (up to identifiability). The standard and mixed membership stochastic block models are special cases in which the latent positions take only $K$ distinct vector values, representing communities, or live in the $(K-1)$-simplex with those vertices, respectively. Under the stochastic block model, our theory suggests spectral clustering using a Gaussian mixture model (rather than $K$-means) and, under mixed membership, fitting the minimum volume enclosing simplex, existing recommendations previously only supported under non-negative-definite assumptions.  Empirical improvements in link prediction (over the random dot product graph), and the potential to uncover richer latent structure (than posited under the standard or mixed membership stochastic block models) are demonstrated in a cyber-security example.
\end{abstract}

\section{Introduction}
While the study of graphs is well-established in Mathematics and Computer Science, it has only more recently become mainstream in Statistics, a shift driven at least in part by the advent of the Internet \citep{newman2018networks}. Yet, despite its breadth, translating existing graph theory into principled statistical procedures has produced many new mathematical challenges.

An example pertinent to this paper is the spectral clustering procedure \citep{von2007tutorial}. This procedure, which aims to find network communities, generally proceeds along the following steps. Given an undirected graph, the corresponding adjacency or normalised Laplacian matrix is first constructed. Next, the graph is \emph{spectrally embedded} into $d$ dimensions by computing the $d$ principal eigenvectors of the matrix --- in our case scaled according to eigenvalue --- to obtain a $d$-dimensional vector representation of each node. The first scaled eigenvector can be thought to provide the $x$-coordinate of each node, the second the $y$-coordinate, and so forth. Finally, these points are input into a clustering algorithm such as $K$-means \citep{steinhaus1956division,lloyd1982least} to obtain communities. The most popular justification for this algorithm, put forward by \citet{shi2000normalized} based on earlier work by \citet{donath1973lower,fiedler1973algebraic}, is its solving a convex relaxation of the normalised cut problem. A more principled statistical justification was finally found by \citet{rohe2011spectral}, see also \citet{lei2015consistency}, showing that the spectral clustering algorithm provides consistent identification of communities under the stochastic block model \citep{holland1983stochastic}. Their analysis however demands that eigenvectors corresponding to the largest \emph{magnitude} eigenvalues are used, countering earlier and contemporary papers which recommend using only the positive. 

The random dot product graph is a model which allows statistical interpretation of spectral embedding as a standalone procedure, i.e., without the subsequent clustering step. Through this broader view of spectral embedding, one finds that geometric analyses other than clustering are also productive. For example, simplex-fitting \citep{rubin2017consistency} and spherical clustering \citep{qin2013regularized,lyzinski2014perfect,lei2015consistency,passino2020spectral}, respectively, are appropriate under the mixed membership \citep{airoldi2008mixed} and degree-corrected \citep{karrer2011stochastic} stochastic block models, and manifold-fitting is appropriate under several other random graph models \citep{athreya2018estimation,trosset2020learning,rubin2020manifold,whiteley2021matrix}. However, the random dot product graph has an important shortcoming, addressed in this paper, which is to make a positive-definite assumption, consistent with the afore-mentioned practice of retaining only the positive eigenvalues, that is problematic for modelling several common types of graph connectivity structure.

The limitations of this positive-definite assumption are easy to identify using a stochastic block model. In standard form, this model posits that there is a partition of the nodes into $K$ communities, conditional upon which the edges occur independently according to a symmetric inter-community edge probability matrix $\bB \in [0,1]^{K \times K}$, known as the block matrix. If the model is extended to include degree correction \citep{karrer2011stochastic}, those probabilities are subject to nodewise scaling, so that for example a node $i$, of community $1$, and a node $j$, of community $2$, form an edge with probability $w_i  w_j\bB_{12}$. If, and only if, the block matrix is non-negative-definite, the random dot product graph can reproduce either model \citep{lyzinski2014perfect}. It will assign a latent position $X_i$ to each node $i$ which, in the standard case, is precisely one of $K$ possible points, each representing a community. Under degree correction, the position instead lives on one of $K$ rays emanating from the origin, with $\lVert X_i \rVert_2 \propto w_i$. In this way, nodes with larger magnitude positions tend to have larger degree.

A positive-definite block matrix is said to reflect homophilic connectivity, in which `birds of a feather flock together'. But to encounter a block matrix that has negative eigenvalues is not unusual. With $K=2$, negative eigenvalues will occur when $\bB_{12} > \bB_{11}, \bB_{22}$, e.g., under heterophilic connectivity, and may occur when $\bB_{11} > \bB_{12} > \bB_{22}$, e.g., under core-periphery connectivity  --- a densely connected core, community 1, with sparsely connected periphery, community 2 \citep{borgatti2000models}. In fact, when $K > 2$, negative eigenvalues may even occur when $\bB_{ii}>\bB_{ij}$ for each $i\neq j$, connectivity which could reasonably be considered homophilic. A reviewer gave the example
\[ \bB = 0.1 \times \left(\begin{array}{ccc} 9 &  0 &  8\\
0 &  6 &  5\\
8 &  5 & 9\end{array}\right),\]
which has one negative and two positive eigenvalues.

In a graph following a stochastic block model, the signs of the principal eigenvalues of the adjacency and normalised Laplacian will correspond to those of $\bB$, up to noise. In this way a graph from a two-community stochastic block model with $\bB_{12} > \bB_{11}, \bB_{22}$ should present two large-magnitude eigenvalues, one positive and one negative, with the others being close to zero. To give a light-hearted real data example, consider the graph of enmities between Harry Potter characters, a publically available dataset \citep{hpdata}. The same graph was previously studied by \citet{mara2020csne}, and more generally several literature studies involve analysis of character networks \citep{labatut2019extraction}. A plot of the eigenvalues of the graph adjacency matrix is shown in Figure~\ref{fig:spectrum}. Two eigenvalues stand out in magnitude, one positive and one negative, and for the purpose of this example the remaining will be treated as noise. 
\begin{figure}
\centering
\includegraphics[width=.8\textwidth]{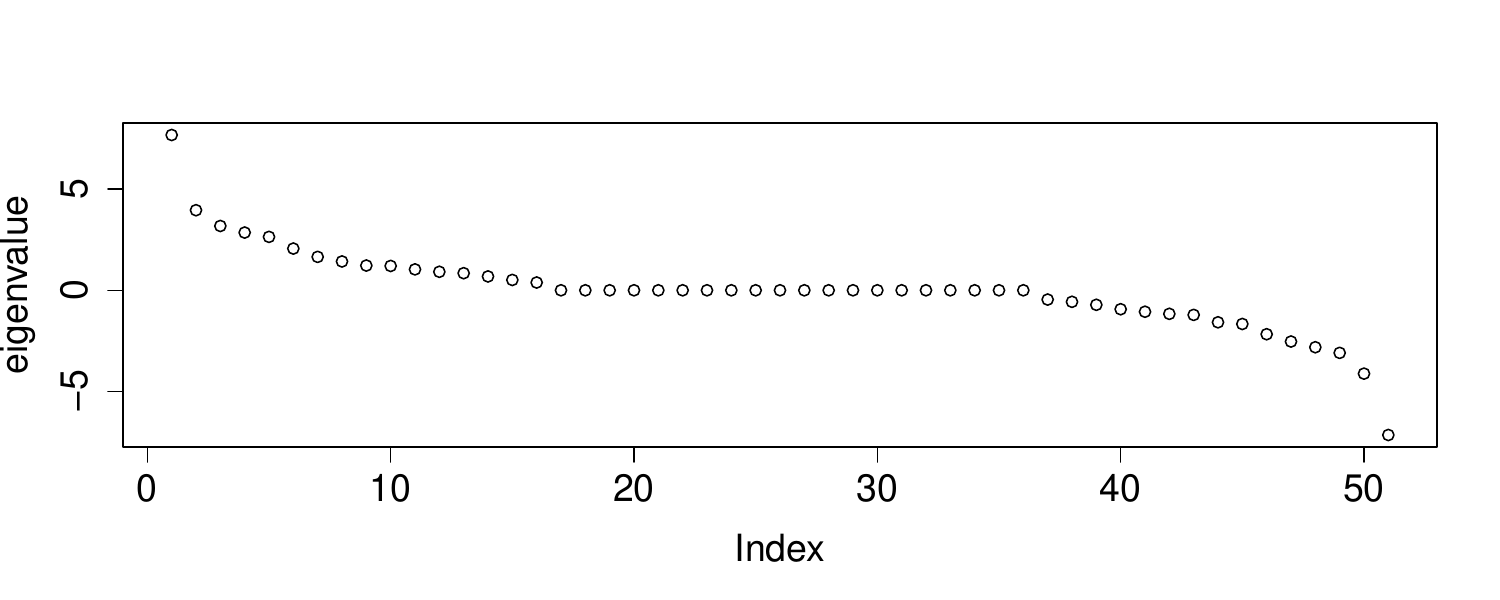}
\caption{Eigenvalues of the adjacency matrix of the graph of enmities between Harry Potter characters. Of the two largest-magnitude eigenvalues, the first is positive and the second negative, and the corresponding eigenvectors are used for spectral embedding in Figure~\ref{fig:harry_potter}.} \label{fig:spectrum}
\end{figure}

Figure~\ref{fig:harry_potter} shows the adjacency spectral embedding of the graph into $2$ dimensions (a formal definition to follow, Definition~\ref{def:ASE_and_LSE}), selecting eigenvectors corresponding to those eigenvalues. One can discern two rays from the origin which, as those familiar with the story will know, distinguish the good characters from the evil. This geometry is precisely what would be expected under a two-community degree-corrected stochastic block model, however it falls outside the scope of the random dot product graph because the second eigenvector, which gives the $y$-axis, has a negative eigenvalue. Upon implementing spherical clustering \citep{lyzinski2014perfect}, we find
\[\hat \bB \propto \left(\begin{array}{cc} 0.05 & 1 \\ 1 &0.09\end{array}\right),\]
which has one positive and one negative eigenvalue, and the colours of the points in the figure reflect the node partition thus obtained (community 1 in red, the `good' characters; community 2 in green, the `evil' characters). The block matrix $\bB$ is not fully identifiable due to the presence of nodewise scaling, but the inter-to-intra community ratios are. In this way, two `good' characters are estimated as 20 times less likely to be enemies than each would with someone `evil', and a similarly low level of enmity between `evil' characters is observed (the difference is not significant). 

\begin{figure}
\centering
\includegraphics[width=\textwidth]{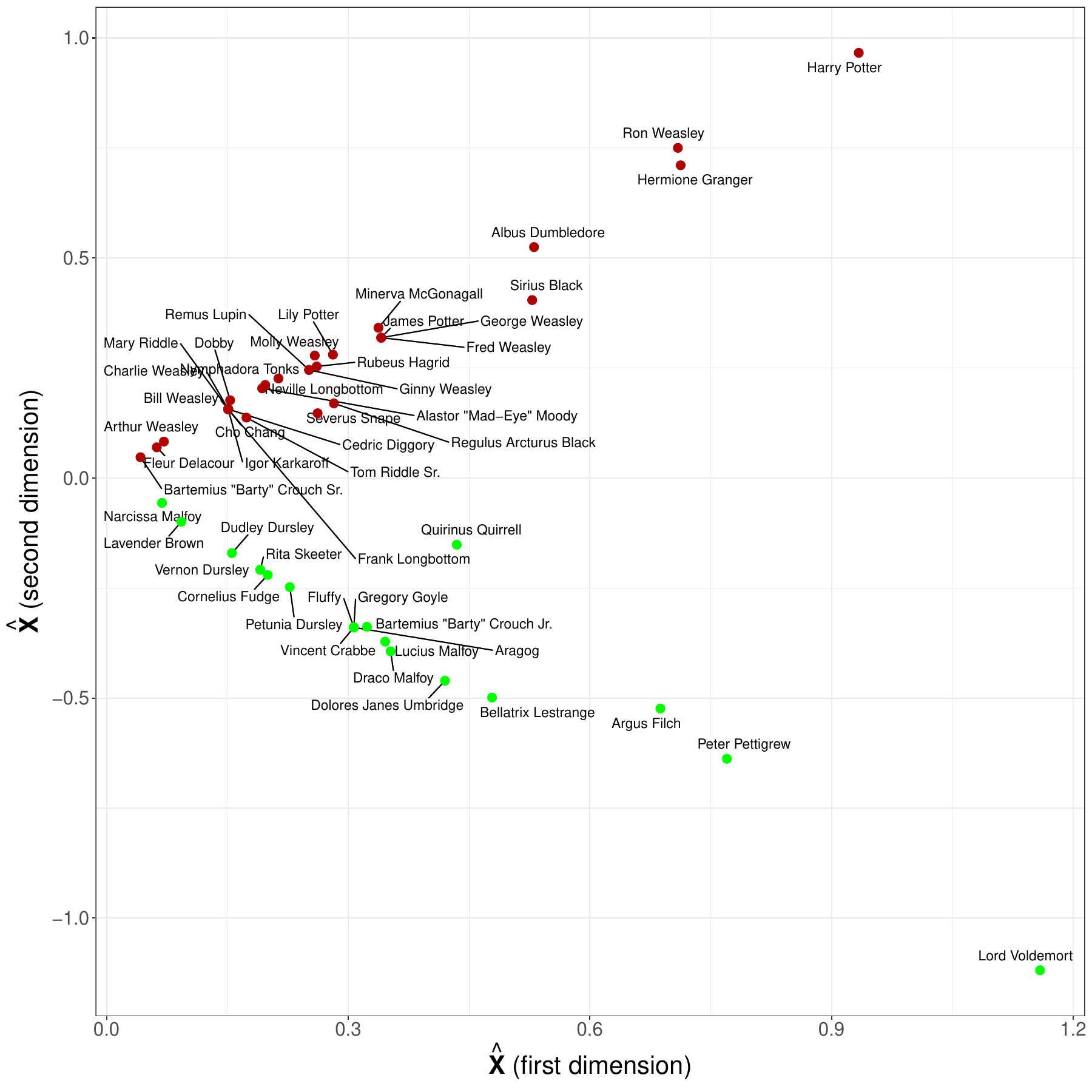}
\caption{Adjacency spectral embedding into $\R^2$ of the graph of enmities between Harry Potter characters. This pattern of two rays emanating from the origin is expected under a two-community degree-corrected stochastic block model, and the points are coloured by their inferred community, estimated using spherical clustering \citep{lyzinski2014perfect}.  The embedding uses the eigenvectors corresponding to the two largest-magnitude eigenvalues, of which the first is positive and the second negative, as shown in Figure~\ref{fig:spectrum}.} \label{fig:harry_potter}
\end{figure}

Moving beyond this toy example, a diversity of real-world graphs are surveyed in Section~\ref{sec:graph_survey}, finding that half present important negative eigenvalues. Following this, a closer study of computer network data is conducted in Section~\ref{sec:cyber-data} giving physical reasons for heterophilic connectivity structure, and showing that including negative eigenvalues improves predictions. A model generalising the random dot product graph to allow for negative eigenvalues is therefore called for.

Our proposed model has the generic structure of a \emph{latent position model} \citep{hoff2002latent}, in that each node $i$ is posited to have a latent position $X_i \in \R^d$, and two nodes $i$ and $j$ form an edge, conditionally independently, with probability $f(X_i, X_j)$, for some function $f$. A Generalised Random Dot Product Graph (GRDPG) is a latent position model with $f(x,y) = x^\top \indefI y$, where $\indefI = \diag(1, \ldots,1,-1, \ldots, -1)$, with $p$ ones followed by $q$ minus ones on its diagonal, and where $p \geq 1$ and $q \geq 0$ are two integers satisfying $p+q=d$. When $q=0$, the function $f$ becomes the usual inner product on $\R^d$, and the model reduces to the standard random dot product graph \citep{nickel06,young2007random,athreya2017statistical}.

The core asymptotic findings of this paper (Theorems~\ref{thrm:GRDPG_ASE_ttinf}--\ref{thrm:GRDPG_LSE_CLT}) mirror existing results for the random dot product graph \citep{sussman2012consistent,lyzinski2014perfect,athreya2016limit,lyzinski2017community,cape_biometrika,cape2017two,tang2016limit}: Whether the adjacency or normalised Laplacian matrix is used, the vector representations of nodes obtained by spectral embedding provide uniformly consistent and asymptotically Gaussian latent position estimates (up to identifiability). In this way, existing recommendations to fit a Gaussian mixture model (rather than $K$-means clustering) for spectral clustering under the stochastic block model \citep{athreya2016limit,tang2016limit}, and simplex-fitting under mixed membership \citep{rubin2017consistency}, previously justified only under non-negative-definite assumptions, now stand in general.

Other than the random dot product graph, there are several precursors to the GRDPG, most notably the eigenmodel of \citet{hoff2008modeling} with kernel $f(x,y) = \Phi(\mu + x^\top \boldsymbol \Lambda y)$ (ignoring covariates), and the proposed random dot product graph generalisation with kernel $f(x,y) = x^\top \boldsymbol \Lambda y$ by \citet{rohe2018note}, where $\boldsymbol \Lambda \in \R^{d \times d}$ is respectively diagonal or symmetric, $\mu$ is a scalar and $\Phi$ is the normal distribution function. Negative eigenvalues in $\boldsymbol \Lambda$ allow us to model negative eigenvalues in the graph adjacency matrix. To our best knowledge those models' connection to spectral embedding is unexplored. While the model by \citet{rohe2018note} evidently absorbs ours, the two are equivalent up to identifiability, and we will discuss how our asymptotic results can be adapted to this larger parameter space in Section~\ref{sec:alternative_parameterisations}. A contemporaneously written paper, \citet{lei2018network}, comes closer to our work, proposing the kernel $f(x,y) = \langle x_1, y_1 \rangle_1 - \langle x_2,y_2 \rangle_2$, where $x = (x_1, x_2)$ and $y = (y_1, y_2)$ live on the direct sum of two Hilbert spaces with respective inner products $\langle \cdot, \cdot \rangle_1$ and $\langle \cdot, \cdot \rangle_2$, and proving the consistency of adjacency spectral embedding in a form of Wasserstein distance. The GRDPG is a special case where the Hilbert spaces are $\R^p$ and $\R^q$, equipped with the Euclidean inner product. The advantage of Lei's analysis is to handle the infinite-dimensional case; on the other hand, our finite-dimensional results are stronger, so much that they lead to concrete methodological recommendations that could not be made based only on Lei's results. Those include to use fit a Gaussian mixture model (rather than $K$-means clustering) for spectral clustering under the stochastic block model (calling on Theorems~\ref{thrm:GRDPG_ASE_CLT} and \ref{thrm:GRDPG_LSE_CLT}), and minimum volume simplex fitting under mixed membership (calling on Theorem~\ref{thrm:GRDPG_ASE_ttinf}). If the latent positions of the GRDPG are independent and identically distributed (i.i.d.), as will be assumed in our asymptotic study, the model also admits an Aldous-Hoover representation \citep{aldous1981representations,hoover1979relations}, wherein each node is instead independently assigned a latent position uniformly on the unit interval, and connections occur conditionally independently according to a kernel $g : [0,1]^2 \rightarrow [0,1]$, known as a graphon \citep{lovasz2012large}. Conversely, an Aldous-Hoover graph follows a GRDPG if the integral operator associated with $g$ has finite rank \citep{lei2018network,rubin2020manifold}. If this operator has negative eigenvalues, it cannot be reproduced by the random dot product graph or any other latent position model with positive-definite kernel.

The rest of this article is organised as follows. In Section~\ref{sec:data_embedding_model}, we formally describe the data envisaged, the spectral embedding procedure, and the problem of finding a model-based rationale for this approach. Then we propose our solution, a generalisation of the random dot product graph model, discussing its identifiability and alternative parameterisations. Section~\ref{sec:spectral_embedding} presents our asymptotic results which, collected, say: spectral embedding provides uniformly consistent estimates of the latent positions of a GRDPG, with asymptotically Gaussian error (up to identifiability). In Section~\ref{sec:implications}, the implications of this theory for standard and mixed membership stochastic block model estimation are discussed, namely, the advantages of fitting a Gaussian mixture model over $K$-means for spectral clustering under the stochastic block model, and the consistency of minimum volume enclosing simplex fitting under mixed membership. In Section \ref{sec:real-data} we review a diversity of real-world graphs, showing that many exhibit important negative eigenvalues, before focussing on a cyber-security application. Section~\ref{sec:conclusion} concludes. All proofs are relegated to the Appendix.

\section{The data, spectral embedding, and model}\label{sec:data_embedding_model}
This paper concerns statistical inference based on a single, observed graph on $n$ nodes, labelled $1, \ldots, n$. In conventional statistical terms, one may view the graph as `the data', and its number of nodes as a loose substitute for `sample size'. The graph is represented by its adjacency matrix $\bA \in \{0,1\}^{n \times n}$, where $\bA_{ij}
= 1$ if and only if there is an edge between the $i$th and $j$th node. The graph is assumed to be undirected with no-self loops or, equivalently, $\bA$ is symmetric ($\bA = \bA^\top$) and hollow ($\bA_{ii}=0$ for all $i$). 

To allow statistical analysis using mainstream methods (e.g., clustering), it is common to seek a vector representation of each node, and spectral embedding is a popular tool for this purpose.

\begin{definition}[Adjacency and Laplacian spectral embedding into $\R^d$]
  \label{def:ASE_and_LSE}
  Let $\hat{\bS}$ be the $d \times d$ diagonal matrix containing the $d$ largest eigenvalues of $\bA$ \emph{in magnitude} on its diagonal, arranged in decreasing order (based on their actual, signed, value), and let $\hat{\bU} \in \R^{n \times d}$ be a matrix containing, as columns, corresponding orthonormal eigenvectors arranged in the same order. Define the adjacency spectral embedding of the graph into $\R^{d}$ as the matrix $\hat{\bX} = [\hat X_{1},\dots, \hat X_{n}]^{\top}= \hat{\bU}|\hat{\bS}|^{1/2} \in \R^{n \times d}$, i.e., $\hat{\bX}$ is a matrix whose $i$th row, transposed into a column vector, is $\hat X_i$.
  Similarly, let  $\bL =\bD^{-1/2}\bA\bD^{-1/2} \in \R^{n \times n}$ denote the normalised Laplacian of the graph, where $\bD \in \R^{n \times n}$ is the degree matrix, a diagonal matrix with $\bD_{ii} = \sum_j \bA_{ij}$ for all $i$, and let $\breve{\bS}, \breve{\bU}$ respectively denote the corresponding matrices of largest-magnitude eigenvalues and associated eigenvectors. Define the Laplacian spectral embedding of the graph into $\R^{d}$ by $\breve{\bX} = [\breve X_{1},\dots, \breve X_{n}]^{\top} = \breve{\bU}|\breve{\bS}|^{1/2}  \in \R^{n \times d}$.  
\end{definition}
The problem considered in this paper is finding a model-based rationale for spectral embedding. We seek a random graph model that defines true latent positions $X_1, \ldots, X_n \in \R^d$ such that $\hat X_i$ provides an estimate of $X_i$ with quantifiable error. This search will also yield a suitable transformation of $X_i$ that may be treated as the estimand of $\breve X_{i}$.

As alluded to in the introduction, a relatively large body of work exists, comprehensively reviewed in \citet{athreya2017statistical}, addressing the same problem with the eigenvalues in Definition~\ref{def:ASE_and_LSE} selected by largest (signed) value, in other words, leaving out negative eigenvalues and corresponding eigenvectors. To interpret such embeddings, a latent position model known as the random dot product graph \citep{nickel06,young2007random} is put forward and, in this model, an edge between two nodes occurs with probability given by the inner product of their latent positions. However, such a model must result in a non-negative-definite edge probability matrix $\bP_{ij} = X_i^\top X_j$, and cannot explain significant negative eigenvalues in $\bA$, because the matrices are related by $\E(\bA \mid \bP) = \bP$, so that any difference between their spectra is due to noise.

Our solution is a model which generalises the random dot product graph, in specifying that the probability of an edge between two nodes is given by the \emph{indefinite} inner product of their latent positions. For two vectors $x, y \in\R^d$, this product is $x^\top \indefI y$, where $\indefI$ is a diagonal matrix with $p$ ones followed by $q$ minus ones on its diagonal, and $p \geq 1$ and $q \geq 0$ are two integers satisfying $p+q=d$. A formal model definition is now given.

\begin{definition}[Generalised random dot product graph model]
	\label{def:GRDPG}
        Let $\mathcal{X}$ be a subset of $\R^{d}$ such that $x^{\top}\indefI y \in [0,1]$ for all $x,y \in \mathcal{X}$, and $\mathcal{F}$ a joint distribution on $\mathcal{X}^n$. We say that $(\bX,\bA )\sim \text{GRDPG}(\mathcal{F})$, with signature $(p,q)$, if the following hold. First, let $(X_1, \ldots, X_n) \sim \mathcal{F}$, to form the latent position matrix $\bX = [X_{1},\dots,X_{n}]^{\top} \in \R^{n \times d}$. Then, the graph adjacency matrix $\bA \in \{0,1\}^{n \times n}$ is symmetric, hollow and, conditional on $X_1, \ldots, X_n$, 
        \begin{equation}
          \bA_{ij} \overset{ind}{\sim} \text{Bernoulli}\left(X_i^\top \indefI X_j\right), \label{eq:grdpg}\end{equation}
        for all $i < j$.
\end{definition}

\subsection{Special cases}
\subsubsection{The stochastic block model} \label{sec:spesh_sbm}
A graph follows a stochastic block model if there is a partition of the nodes into $K$ communities, conditional upon which $\bA_{ij}\overset{ind}\sim \text{Bernoulli}(\bB_{Z_i Z_j})$, for $i < j$, where $\bB \in [0,1]^{K \times K}$ is symmetric and $Z_i \in \{1,\ldots,K\}$ is an index denoting the community of the $i$th node.

Let $p \geq 1, q \geq 0$ denote the number of strictly positive and strictly negative eigenvalues of $\bB$ respectively, put $d = p+q$, and choose $\rv_1, \ldots, \rv_K \in \R^d$ such that $\rv_k^{\top} \indefI \rv_l = \bB_{kl}$, for $k,l \in \{1,\ldots,K\}$. We will take as a canonical choice the $K$ rows of $\bU_{\bB}|\bSigma_{\bB}|^{1/2}$, where $\bSigma_{\bB}$ is diagonal containing the $d$ non-zero eigenvalues of $\bB$, and $\bB$ has spectral decomposition $\bB= \bU_{\bB}\bSigma_{\bB}\bU_{\bB}^\top$. It may help to remember that $p+q = d = \rank(\bB) \leq K$. By letting $X_i = \mathrm{v}_{Z_i}$, we find that the graph is a GRDPG, and can set $\mathcal{X} = \{\rv_1, \ldots, \rv_K\}$.

\subsubsection{The mixed membership stochastic block model} \label{sec:spesh_mmsbm}
Now, assign (at random or otherwise) to the $i$th node a probability vector $\pi_i \in \mathbb{S}^{K-1}$ where $\mathbb{S}^m$ denotes the standard $m$-simplex. Conditional on this assignment, let
\[\bA_{ij} \overset{ind}{\sim} \text{Bernoulli}\left(\bB_{Z_{i\rightarrow j} Z_{j\rightarrow i}}\right),\]
where 
\[Z_{i\rightarrow j} \overset{ind}{\sim} \text{categorical}(\pi_i) \quad \text{and}\quad  Z_{j\rightarrow i} \overset{ind}{\sim} \text{categorical}(\pi_j),\]
for $i < j$. The resulting graph is said to follow a mixed membership stochastic block model \citep{airoldi2008mixed}.

Averaging over $Z_{i\rightarrow j}$ and $Z_{j \rightarrow i}$, we can equivalently write that, conditional on $\pi_1, \ldots, \pi_n$, 
\[\bA_{ij} \overset{ind}{\sim} \text{Bernoulli}\left(\pi_i^\top \bB \pi_j\right).\]
But if $p,q,d$ and $\rv_1, \ldots, \rv_K$ are as defined previously, then $\pi_i^\top \bB \pi_j = (\sum_{k=1}^K \pi_{ik} \rv_k^\top) \indefI (\sum_{k=1}^K \pi_{jk} \rv_k) = X_i^\top \indefI X_j$, where $X_i = \sum \pi_{ik} \rv_k$. Therefore, conditional on $X_1, \ldots, X_n$, Equation~\eqref{eq:grdpg} holds, and the graph is a GRDPG with latent positions $X_1, \ldots, X_n$. These live in the convex hull of $\rv_1, \ldots, \rv_K$, a $(K-1)$-simplex if $\bB$ has full rank ($d=K$).

\begin{figure}
\centering
\includegraphics[width=7cm]{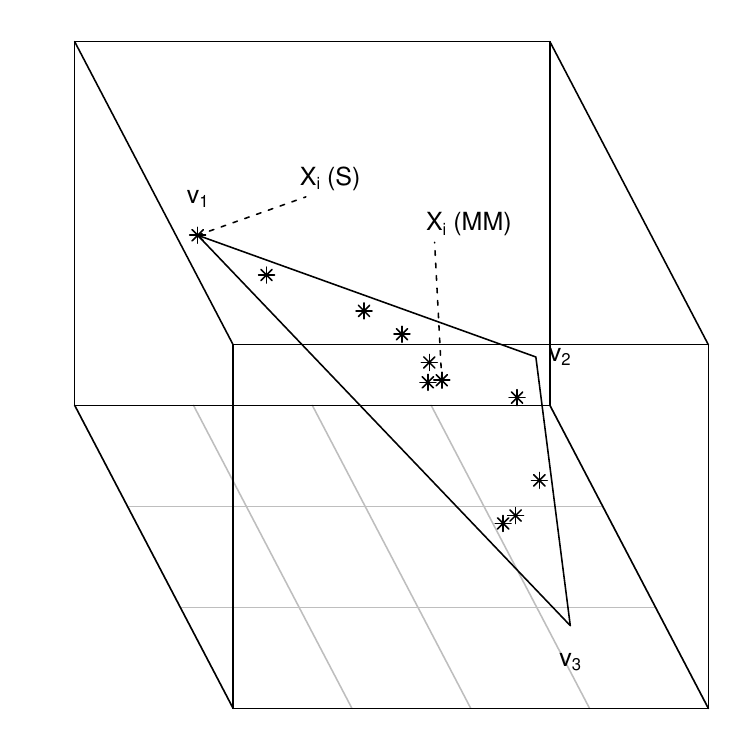}
\caption{Illustration of standard (S) and mixed membership (MM) stochastic block models as special cases of the GRDPG model (Definition~\ref{def:GRDPG}). The models have $K=3$ communities and so the corresponding GRPDG will require $d=3$ dimensions (or fewer, if the block matrix has low rank). The points $\rv_1, \ldots, \rv_K$ represent communities. Under the standard stochastic block model, the $i$th node is assigned to a single community so that $X_i \in \{\rv_1, \ldots, \rv_K\}$. Under mixed membership, if the $i$th node has a community membership probability vector $\pi_i$, then its position in latent space, $X_i$, is the corresponding convex combination of $\rv_1, \ldots, \rv_K$, that is, $X_i = \sum_{k=1}^K \pi_{ik} \rv_k$.} \label{fig:mmsbm_illustration}
\end{figure}

The GRDPG model therefore gives the standard and mixed membership stochastic block models a natural spatial representation in which $\rv_1, \ldots, \rv_K$ represent communities, and latent positions in between them represent nodes with mixed membership. This is illustrated in Figure~\ref{fig:mmsbm_illustration}.

\citet{airoldi2008mixed} set $\pi_1, \ldots, \pi_n \overset{i.i.d.} \sim \text{Dirichlet}(\alpha)$ for some  $\alpha \in \R_+^K$. The corresponding latent positions $X_1, \ldots, X_n$ are then a) also i.i.d., and b) fully supported on the convex hull of $\rv_1, \ldots, \rv_K$. Consistency of simplex-fitting, Algorithm~\ref{alg:mmspec} (Section~\ref{sec:spectral_embedding}) and illustrated in Figure~\ref{fig:SBM_MMSBM_superplot}e), relies on these two points, without requiring a Dirichlet distribution assumption.

\subsubsection{The degree-corrected stochastic block model}
Instead, assign (at random or otherwise) to the $i$th node a weight $w_i \in [0,1]$. Conditional on this assignment, let
\[\bA_{ij} \overset{ind}{\sim} \text{Bernoulli}\left(w_i w_j \bB_{Z_i Z_j}\right),\]
for $i < j$. The resulting graph is said to follow a degree-corrected stochastic block model \citep{karrer2011stochastic}.

With $p,q,d$ and $\rv_1, \ldots, \rv_K$ as defined previously, a corresponding GRDPG is constructed by letting $X_i = w_i \mathrm{v}_{Z_i}$, which lives on one of $K$ rays emanating from the origin.

\subsection{Identifiability}\label{sec:identifiability}
In the definition of the GRDPG, it is clear that the conditional distribution of $\bA$ given $X_1, \ldots, X_n$ would be unchanged if $X_1, \ldots, X_n$ were replaced by  $\bQ X_1, \dots, \bQ X_n$, for any matrix $\bQ \in \indefO = \{\bM \in \R^{d \times d}\ :\ \bM^{\top} \indefI \bM = \indefI\}$, known as the indefinite orthogonal group. The vectors $X_1, \ldots, X_n$ are therefore identifiable from $\bA$ only up to such transformation.

The property of identifiability up to \emph{orthogonal} transformation, that is, by a matrix $\bW \in \mathbb{O}(d) = \{\bM \in \R^{d \times d}\ :\ \bM^{\top} \bM = \bI\}$ is encountered in many statistical applications and occurs when $q = 0$. This unidentifiability property often turns out to be moot since inter-point distances are invariant under the action of a common orthogonal transformation, and many statistical analyses (such as $K$-means clustering) depend only on distance. When $q>0$, the transformation is \emph{indefinite orthogonal} and can affect inter-point distances. This is illustrated in Figure~\ref{fig:identifiability} with a GRDPG of signature $(1,2)$. The group  $\defO(1,2)$ contains rotation matrices
\begin{equation*}
r_{t} = \left[\begin{array}{ccc}1 & 0 & 0 \\ 0 & \cos t & -\sin t \\ 0 & \sin t & \cos t \end{array}\right],
\end{equation*}
but also hyperbolic rotations
\begin{equation*}
\rho_{\theta} = \left[\begin{array}{ccc}\cosh \theta & \sinh \theta & 0 \\ \sinh \theta & \cosh \theta & 0 \\ 0 & 0 & 1 \end{array}\right],
\end{equation*}
as can be verified analytically. A rotation $r_{\pi/3}$ is applied to three GRDPG latent positions to get from the top-left to the top-right panel in Figure~\ref{fig:identifiability}. Hyperbolic rotations $\rho_{\theta}$ ($\theta = 1.3$, chosen arbitrarily) and $\rho_{-\theta}$ take the positions from the top-left to the bottom-left and from the top-right to the bottom-right panels, respectively. These transformations alter inter-point distances: in the bottom row, the blue position is closer to the green on the left and closer to the red on the right; the three positions are equidistant in the top row.

\begin{figure}
\centering
\includegraphics[width=7cm]{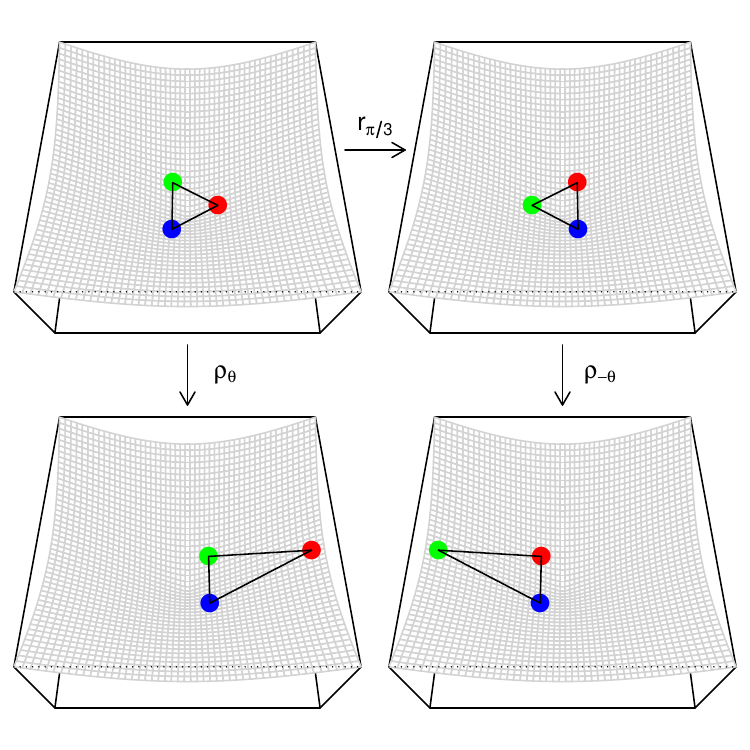}
\caption{Identifiability of the latent positions of a GRDPG with signature $(1,2)$. In each panel, the three coloured points represent latent positions $X_1$, $X_2$ and $X_3$, which live inside the cone $\{x \in \R^3: x^\top \bI_{1,2} x = 0\}$ (grey mesh). The positions are only identifiable up to transformation by a matrix in the indefinite orthogonal group $\mathbb{O}(1,2)$. This group includes some rotations (e.g., that used to go from the top-left to top-right panel), but also hyberbolic rotations (e.g., going from top-left to bottom-left and top-right to bottom-right). The observed graph is equally likely under those four latent position configurations, and so the configurations cannot be distinguished. Some of these transformations affect inter-point distances. In the bottom row, the blue position is closer to the green on the left, whereas it is closer to the red on the right; the three positions are equidistant in the top row.}\label{fig:identifiability}
\end{figure}

\subsection{Alternative parameterisations}\label{sec:alternative_parameterisations}
In this section we discuss alternative, equivalent parameterisations, explaining why we opted for the GRDPG without claiming objective superiority.

It may be observed that the only ambiguity in computing the spectral embedding $\hat{\bX}$ is how the principal eigenvectors for $\bA$ are chosen. If we assume repeated non-zero eigenvalues are rare in real data, this choice in practice is typically limited to the option of reversing any eigenvector. A model for interpreting spectral embedding might have been expected to reflect only this kind of ambiguity in its unidentifiability.

The model structure can indeed be brought closer to the spectral decomposition of $\bA$ by defining an alternative, `spectral' estimand, $\tilde{\bX} = [\tilde X_{1},\dots, \tilde X_{n}]^{\top}= \bU|\bS|^{1/2} \in \R^{n \times d}$, where $\bS\in\R^{d \times d}$ is a diagonal matrix containing the non-zero eigenvalues of $\bP = \bX \indefI \bX^\top$, in decreasing order, and $\bU \in \R^{n \times d}$ contains corresponding orthonormal eigenvectors as columns. With the help of a follow-on paper \citep{agterberg2020nonidentifiability}, the vector $\hat X_i$ will be found to estimate $\tilde X_{i}$ \emph{up to orthogonal transformation}, uniformly and with asymptotically Gaussian error, or, under a distinct eigenvalue assumption, \emph{up to reflection of the axes} (see Section~\ref{sec:results}).

The object $\tilde{\bX}$ has special structure which, for example, precludes its rows $\tilde X_1, \ldots, \tilde X_n$ from being i.i.d.; under a stochastic block model, the $K$ unique vector values taken by $\tilde X_1, \ldots, \tilde X_n$, representing the communities, cannot be determined from only $\bB$, because they depend on $Z_1, \ldots, Z_n$ --- so, for example, those $K$ vectors may change as $n$ grows. These are some reasons why we choose to model $\tilde{\bX}$ through $\bX$ (whence $\bP$), rather than as a standalone object. However, a reader only interested in estimating $\tilde X_{i}$ can ignore indefinite orthogonal transformations, which only appear when we try to relate $\tilde X_{i}$ to $X_i$.

A convention could be imposed on $\mathcal{F}$ to make $X_i$ and $\tilde X_{i}$ converge to each other, up to orthogonal transformation, and proposals to this effect were made in a follow-on paper \citep{agterberg2020nonidentifiability}. However, convergence is not fast enough, when it comes to the central limit theorem, to substitute $\tilde X_{i}$ by $X_i$ and avoid indefinite orthogonal transformations (see Section~\ref{sec:discussion}). Moreover, under such a convention, the construction of vector representatives of the $K$ communities under the stochastic block model and extensions is more involved, compared to our canonical construction in Section~\ref{sec:spesh_sbm}, and will vary depending on the distributions of $Z_1, \ldots, Z_n$, degree-correction weights $w_1, \ldots, w_n$, and community membership probabilities $\pi_1, \ldots, \pi_n$. 

In the context of graph simulation, \citet{rohe2018note} proposed to generalise the random dot product graph via a latent position model with kernel  $f(x,y) = x^\top \boldsymbol \Lambda y$, where $\boldsymbol \Lambda \in \R^{d \times d}$ is a symmetric matrix. Alternatively, in the spirit of the eigenmodel of \citet{hoff2008modeling}, one might consider enforcing $\boldsymbol \Lambda$ to be diagonal. In either case, the model's latent positions, $Y_i$ say, can be transformed into the latent positions of an equivalent GRDPG with signature $(p,q)$, via $X_i = \mathbf{L} Y_i$, given a decomposition $\boldsymbol{\Lambda} = \mathbf{L}^{\top} \indefI \mathbf{L}$ for some $p,q$ (such as that proposed for $\bB$ in Section~\ref{sec:spesh_sbm}). If $\boldsymbol \Lambda$ is only assumed symmetric and the $Y_i$ restricted only to give valid probabilities (i.e. to belong to a set $\mathcal{Y}$ in which $x^{\top}\boldsymbol \Lambda y \in [0,1]$ for all $x,y \in \mathcal{Y}$), then $Y_i$ are identifiable only up to invertible linear transformation, since we can replace $Y_1, \ldots, Y_n$ with $\mathbf{M} Y_1, \ldots, \mathbf{M} Y_n$ and $\boldsymbol \Lambda$ with $\mathbf{M}^{-\top} \boldsymbol \Lambda \mathbf{M}^{-1}$, for any invertible matrix $\mathbf{M}$, without changing the conditional distribution of $\bA$. If the model of \citet{rohe2018note} is preferred, the results of Section~\ref{sec:results} can be re-interpreted to say that $\hat X_i$ estimates $Y_i$ \emph{up to invertible linear transformation}, uniformly and with asymptotically Gaussian error: in the theorems of Section~\ref{sec:results} we would replace $\bQ \hat X_i$ with $\mathbf{L}^{-1}\bQ \hat X_i$ (assuming $\boldsymbol \Lambda$ has full rank), $X_i$ with $Y_i$, updating the covariance matrices accordingly.

\section{Asymptotics}\label{sec:spectral_embedding}
This section describes the asymptotic statistical properties of GRDPG latent position estimates obtained by spectral embedding. 
\subsection{Results for adjacency spectral embedding}\label{sec:results}
In this section, the spectral estimates $\hat X_1, \ldots, \hat X_n$ are shown to converge to $X_1, \ldots, X_n$ in two standard statistical senses: uniformly, and with asymptotically Gaussian error. Analogous results for Laplacian spectral embedding are given in Section~\ref{sec:lse_results}.

For a given $n$, the latent positions are assumed to be independent and identically distributed. As $n \rightarrow \infty$, their distribution is either fixed or, to produce a regime in which the average node degree grows less than linearly in $n$, it is made to shrink. This is done by letting $X_i = \rho^{1/2}_n\xi_i$, where $\xi_i \overset{i.i.d.}{\sim} F$, for some distribution $F$ on $\R^d$, and allowing the cases $\rho_n = 1$ or $\rho_n \rightarrow 0$ sufficiently slowly. The generic joint distribution $\mathcal{F}$ occurring in Definition~\ref{def:GRDPG} is therefore assumed to factorise into a product of $n$ identical marginal distributions that are equal to $F$ up to scaling. The dimension of the model, $d$, is assumed to have been chosen `economically' in the sense that, for $\xi \sim F$, the second moment matrix $\bDelta = \mathbb{E}(\xi \xi^{\top}) \in \R^{d \times d}$ has full rank. Here $d$ is viewed as fixed and known, so for simplicity we suppress $d$-dependent factors in the statements of our theorems. Our proofs, however, keep track of $d$.

Since the average node degree grows as $n \rho_n$, the cases $\rho_n = 1$ and $\rho_n \rightarrow 0$ can be thought to respectively produce dense and sparse regimes and $\rho_n$ is called a sparsity factor. No algorithm can produce uniformly consistent estimates of $X_1, \ldots, X_n$ if the average node degree grows less than logarithmically. Indeed, if one did, it could be used to break the information-theoretic limit for perfect community recovery under the stochastic block model \citep{abbe2017community}. Our results hold under polylogarithmic growth.

Recall that we have also defined a `spectral' estimand, $\tilde X_i$, identifiable up to orthogonal rather than indefinite orthogonal transformation (see Section~\ref{sec:alternative_parameterisations}). To move between $\hat X_i$, $\tilde X_i$ and $X_i$, we introduce the transformations:
\begin{enumerate}
\item $\bW_{\star} \in \defO(d) \bigcap\indefO$: a block orthogonal matrix of which the first $p \times p$ block (respectively second $q \times q$ block) aligns the first $p$ (respectively second $q$) columns of $\bU$ with the first $p$ (respectively second $q$) columns of $\hat \bU$, solving the two orthogonal Procrustes problems independently (explicit construction in the Appendix).
\item $\bQ_{\bX} \in \indefO$: an indefinite orthogonal matrix solving $\bX = \tilde{\bX}\bQ_{\bX}$.
\end{enumerate}
We will find that $\bW_{\star} \hat X_i$ converges to $\tilde{X}_i$ and that $\bQ \hat X_i$ converges to $X_i$, where  $\bQ = \bQ_{\bX}^\top \bW_{\star}$.
\begin{theorem}[Uniform consistency of adjacency spectral embedding]
  \label{thrm:GRDPG_ASE_ttinf}
  There exists a universal constant $c > 1$ such that, provided the sparsity factor satisfies $n\rho_{n} =\omega\{\log^{4c} n\}$,
  \begin{equation*}
      \underset{i \in \{1, \ldots, n\}}{\textnormal{max}}\|\mathbf{W}_{\star} \hat X_i - \tilde X_i\|  = \probO\left(\tfrac{\log^c n}{n^{1/2}}\right); \quad
  \underset{i \in \{1, \ldots, n\}}{\textnormal{max}}\|\mathbf{Q} \hat X_i - X_i\|  = \probO\left(\tfrac{\log^c n}{n^{1/2}}\right).
\end{equation*}
\end{theorem}
We say that a random variable $Y$ is $\probO(f(n))$ if, for any positive constant $c>0$ there exists an integer $n_{0}$ and a constant $C>0$ (both of which possibly depend on $c$) such that for all $n \ge n_{0}$, $|Y|\le Cf(n)$ with probability at least $1-n^{-c}$. We write that sequences $a_n = \omega(b_n)$ when there exist a positive constant $C$ and an integer $n_0$ such that $a_n \geq C b_n$ for all $n \geq n_0$ and $a_n/b_n \rightarrow \infty$.

As the graph grows, we will now look at a fixed, finite subset of the nodes, indexed $1, \ldots, m$ without loss of generality, to obtain a central limit theorem on the corresponding errors.
\begin{theorem}[Adjacency spectral embedding central limit theorem]
  \label{thrm:GRDPG_ASE_CLT}
  Assume the same sparsity conditions as Theorem~\ref{thrm:GRDPG_ASE_ttinf}.  Conditional on $X_i$, for $i = 1, \ldots, m$, the random vectors $\sqrt{n} (\mathbf{Q} \hat X_i - X_i) = \sqrt{n} \mathbf{Q}_{\mathbf{X}}^\top(\mathbf{W}_{\star} \hat X_i - \tilde X_i)$ converge in distribution to independent zero-mean Gaussian random vectors with covariance matrix $\bSigma(\xi_i)$ respectively, where
  \begin{equation*}
    \bSigma(x) =
	    \begin{cases} 
	    \indefI\bDelta^{-1}\Ex[(x^{\top}\indefI \xi)(1-x^{\top}\indefI \xi)\xi\xi^{\top}]\bDelta^{-1}\indefI & \textnormal{if } \rho_{n}=1, \\
	    \indefI\bDelta^{-1}\Ex[(x^{\top}\indefI \xi)\xi\xi^{\top}]\bDelta^{-1}\indefI & \textnormal{if } \rho_{n} \rightarrow 0.
	    \end{cases}
  \end{equation*}
\end{theorem}
The matrix $\mathbf{Q}_{\mathbf{X}}$ has, since the first edition of this paper, been shown to converge \citep{agterberg2020nonidentifiability}, in the sense that for each $n$ one can construct an orthogonal matrix $\mathbf{W}_{\star\star}$ such that $\mathbf{W}_{\star\star} \mathbf{Q}_{\mathbf{X}} \rightarrow \bQ_0$ almost surely, where $\bQ_0 \in \indefO$ is a fixed matrix, made explicit in \citet{agterberg2020nonidentifiability}. The second of these observed that this made a central limit theorem ``up to orthogonal transformation'' possible: in the above, we can replace  $\sqrt{n} (\mathbf{Q} \hat X_i - X_i)$ with $\sqrt{n}\mathbf{W}_{\star\star}(\mathbf{W}_{\star}\hat X_i - \tilde X_i)$ and  $\bSigma(\xi_i) $ with $\bQ^{-\top}_0 \bSigma(\xi_i) \bQ_0^{-1}$. Moreover, if the eigenvalues of $\Delta \indefI$ are distinct, the matrices $\mathbf{W}_{\star}$ and $\mathbf{W}_{\star\star}$ can be taken to be diagonal with entries $1$ or $-1$, reflecting that any eigenvector can be reversed in the spectral decompositions of $\bP$ and $\bA$.

\begin{remark}[Proof overview]
  \label{rem:ASE_proof_overview}
  Theorems~\ref{thrm:GRDPG_ASE_ttinf}~and~\ref{thrm:GRDPG_ASE_CLT} are proved in succession within a unified framework. The proof begins with a collection of matrix perturbation decompositions which eventually yield the relation
  \begin{equation*}
    \hat{\bU}|\hat{\bS}|^{1/2}
    = \bU|\bS|^{1/2}\bW_{\star}
    + (\bA-\bP)\bU|\bS|^{-1/2}\bW_{\star}\indefI + \bR
  \end{equation*}
  for some residual matrix $\bR \in \R^{n \times d}$. 
  Appropriately manipulating the above display equation subsequently yields the important identity
  \begin{equation*}
    n^{1/2}(\hat{\bX}\bW_{\star}^{\top}\bQ_{\bX} - \bX)
    = n^{1/2}(\bA-\bP)\bX(\bX^{\top}\bX)^{-1}\indefI + n^{1/2}\bR\bW_{\star}^{\top}\bQ_{\bX},
  \end{equation*}
  Theorem~\ref{thrm:GRDPG_ASE_ttinf} is then established by bounding the maximum Euclidean row norm (equivalently, the two-to-infinity norm \citep{cape2017two}) of the right-hand side of the above display equation sufficiently tightly. Theorem~\ref{thrm:GRDPG_ASE_CLT} is established with respect to the same transformation $\bQ$ by showing that, conditional on the $i$th latent position, i.e., $i$th row of $\bX$, the classical multivariate central limit theorem can be invoked for the $i$th row of the matrix $n^{1/2}(\bA-\bP)\bX(\bX^{\top}\bX)^{-1}\indefI$, whereas the remaining residual term has vanishing two-to-infinity norm. The technical tools involved include a careful matrix perturbation analysis involving an infinite matrix series expansion of $\hat{\bU}$, probabilistic concentration bounds for $(\bA-\bP)^{k}\bU, 1\le k \le \log n$, delicately passing between norms, and indefinite orthogonal matrix group considerations.
	
	The joint proof of Theorems~\ref{thrm:GRDPG_ASE_ttinf}~and~\ref{thrm:GRDPG_ASE_CLT} captures the novel techniques and necessary additional considerations for moving beyond random dot product graphs considered in previous work to generalised random dot product graphs. The proofs of Theorem~\ref{thrm:GRDPG_LSE_ttinf}~and~Theorem~\ref{thrm:GRDPG_LSE_CLT} (for Laplacian spectral embedding), while laborious, follow \emph{mutatis mutandis} by applying the aforementioned proof considerations within the earlier work and context of the Laplacian spectral embedding limit theorems proven in \citet{tang2016limit}. For this reason, we elect to state those theorems without proof.
\end{remark}

\subsection{Discussion} \label{sec:discussion}
Could we remove all notion of indefinite orthogonal transformation from our results? The answer is yes if we consider only the `spectral' estimand, $\tilde X_i$, but we have not found a way of describing $\hat X_i$ as asymptotically ``Gaussian with mean $X_i$'', where $X_i$ are i.i.d., regardless of sparsity, without invoking indefinite orthogonal transformations. Among equivalent latent position distributions --- equal up to indefinite orthogonal transformation (by push-forward) --- we can choose $F$ such that $\mathbf{W}_{\star\star} \tilde X_i$ and $X_i$ are asymptotically equal \citep{agterberg2020nonidentifiability}. However, the matrix $\mathbf{Q}_{\mathbf{X}}$ does not appear to converge faster than $n^{-1/2}$. As a result, however we choose $F$, the error between $\mathbf{W}_{\star\star}\tilde X_i$ and $X_i$ may not vanish when scaled by $\sqrt{n}$.

It is perhaps remarkable how often the presence of indefinite orthogonal transformation will turn out not to matter. First, a follow-on inference procedure, e.g. for cluster analysis, may happen to be invariant to such a transformation of its input data. A key example is fitting a Gaussian mixture model, which we will shortly discuss in more detail. Second, even if the follow-on procedure is not invariant, it may still be consistent. Indeed, there is nothing in our results disputing the consistency of spectral clustering using $K$-means clustering \citep{rohe2011spectral}. Our uniform consistency result allows us to reprove this, and in the same movement prove the consistency of simplex-fitting (Section~\ref{sec:mmsbm_illustration}) under the mixed membership stochastic block model, given some control on the behaviour of $\bQ$, which we now provide.

\begin{lemma}
  \label{lem:indefiniteSpectralBoundGRDPG}
  The matrix $\mathbf{Q}_{\mathbf{X}}$ has bounded spectral norm almost surely.
\end{lemma}
The same can be said of $\mathbf{Q}$, since $\bW_{*}$ is orthogonal, and of $\mathbf{Q}^{-1} = \indefI \bQ^\top \indefI$.

\begin{proof}[Proof of Lemma~\ref{lem:indefiniteSpectralBoundGRDPG}]
 The matrices $\bS$ and $\bX\indefI\bX^{\top}$ have common spectrum by definition which is further equivalent to the spectrum of $\bX^{\top}\bX\indefI$, since for any conformable matrices $\bM_{1}, \bM_{2}$, $\textrm{spec}(\bM_{1}\bM_{2}) = \textrm{spec}(\bM_{2}\bM_{1})$, excluding zero-valued eigenvalues. By the law of large numbers, $(n\rho_n)^{-1}(\bX^{\top}\bX) \rightarrow \Ex(\xi\xi^{\top})$ almost surely, and so $(n\rho_n)^{-1}(\bX^{\top}\bX\indefI) \rightarrow \Ex(\xi\xi^{\top})\indefI$. It follows that both $(n\rho_n)^{-1}\|\bX^{\top}\bX\|$ and $(n\rho_{n})^{-1}\textrm{min}_{i}|\bS_{ii}|$ converge to positive constants almost surely as $n\rightarrow\infty$. 
 
 Now for $\bQ_{\bX}$ as in the hypothesis, with respect to Loewner order $\bQ_{\bX}^{\top}(\textrm{min}_{i}|\bS_{ii}|\bI)\bQ_{\bX} \le \bQ_{\bX}^{\top}|\bS|\bQ_{\bX}$, where $\bQ_{\bX}^{\top}|\bS|\bQ_{\bX} = \bX^{\top}\bX$. Hence, $\textrm{min}_{i}|\bS_{ii}|\|\bQ_{\bX}\|^{2}=\|\bQ_{\bX}^{\top}(\textrm{min}_{i}|\bS_{ii}|)\bQ_{\bX}\|\le\|\bX^{\top}\bX\|$, from which the claim follows.
\end{proof}

Apart from converging slowly, the limiting $\mathbf{Q}_{\mathbf{X}}$ depends on $F$. Thus, a typical situation where the presence of indefinite orthogonal transformation \emph{does} matter is when comparing two graphs under a null hypothesis in which the latent position distributions $F_1 \neq F_2$. In the Appendix, we provide two, two-graph examples, following standard and degree-corrected stochastic block models respectively, where the block matrices are equal, but the community sizes or degree distributions differ.

\subsection{Results for Laplacian spectral embedding}\label{sec:lse_results}
Analogous results are now given for the case of Laplacian spectral embedding. Here, the estimand is defined as
\[\tfrac{X_i}{\sqrt{\sum_{j}{X_i^{\top}}\indefI X_j}},\]
a latent position normalised according to its expected degree. As before, the estimate $\breve X_{i}$ will only resemble its estimand after indefinite orthogonal transformation by a matrix $\breve{\mathbf{Q}}  \in \indefO$, constructed in an analogous fashion to $\mathbf{Q}$. To avoid more definitions and notation, we will forego defining a `spectral' estimand, as we did with adjacency spectral embedding. 

\begin{theorem}[Uniform consistency of Laplacian spectral embedding]
    \label{thrm:GRDPG_LSE_ttinf}
    Assume the same sparsity conditions as Theorem~\ref{thrm:GRDPG_ASE_ttinf}.  Then,
    	\begin{equation*}
	\underset{i \in \{1, \ldots, n\}}{\textnormal{max}}\left\|\breve{\mathbf{Q}}\breve X_i - \tfrac{X_i}{\sqrt{\sum_{j}{X_i^{\top}}\indefI X_j}}\right\|  = \probO\left(\tfrac{\log^c n}{n \rho_{n}^{1/2}}\right).
	\end{equation*}
\end{theorem}

\begin{theorem}[Laplacian spectral embedding central limit theorem]
  \label{thrm:GRDPG_LSE_CLT}
  Assume the same sparsity conditions as Theorem~\ref{thrm:GRDPG_ASE_ttinf}. Conditional on $X_i$, for $i = 1, \ldots, m$, the random vectors
  \[n \rho_{n}^{1/2} \left(\breve{\mathbf{Q}}\breve X_i-\tfrac{X_i}{\sqrt{\sum_{j}{X_i^{\top}}\indefI X_j}}\right),\]
  converge in distribution to independent zero-mean Gaussian random vectors with covariance matrix $\breve{\bSigma}(\xi_i)$ respectively, where
  \begin{equation*}
    \breve{\bSigma}(x) = 
    \begin{cases}
		\indefI\breve{\bDelta}^{-1}
	    \Ex\left\{
	      \left(\tfrac{x^{\top} \indefI \xi(1- x^{\top} \indefI \xi)}{x^{\top} \indefI \mu }\right)
	      \left(\tfrac{\xi}{\mu^{\top}\indefI \xi} - \tfrac{\breve{\bDelta}\indefI x}{2\mu^{\top}\indefI x}\right)
	      \left(\tfrac{\xi}{\mu^{\top}\indefI \xi} - \tfrac{\breve{\bDelta}\indefI x}{2\mu^{\top}\indefI x}\right)^{\top}\right\}\breve{\bDelta}^{-1}\indefI,
	      & \textnormal{if } \rho_{n}=1,\\
	    \indefI\breve{\bDelta}^{-1}
	    \Ex\left\{
	    \left(\tfrac{x^{\top} \indefI \xi}{x^{\top} \indefI \mu }\right)
	    \left(\tfrac{\xi}{\mu^{\top}\indefI \xi} - \tfrac{\breve{\bDelta}\indefI x}{2\mu^{\top}\indefI x}\right)
	    \left(\tfrac{\xi}{\mu^{\top}\indefI \xi} - \tfrac{\breve{\bDelta}\indefI x}{2\mu^{\top}\indefI x}\right)^{\top}\right\}\breve{\bDelta}^{-1}\indefI
	    & \textnormal{if } \rho_{n} \rightarrow 0,
	    \end{cases}
  \end{equation*}
  and  $\mu = \Ex(\xi)$, $\breve{\bDelta}=\Ex\left(\tfrac{\xi\xi^{\top}}{\mu^{\top}\indefI \xi}\right)$.
\end{theorem}

\section{Implications for stochastic block model estimation}\label{sec:implications}
The asymptotic results of Section~\ref{sec:spectral_embedding} justify the use of the following high-level algorithms \citep{athreya2016limit,tang2016limit,rubin2017consistency} for standard and mixed membership stochastic block model estimation, previously only formally supported under non-negative-definite assumptions on the block matrix $\bB$. 
\begin{algorithm}[H]
\caption{Fitting a stochastic block model (spectral clustering)}\label{alg:spec}
\begin{algorithmic}[1]
  \Statex \textbf{input} adjacency matrix $\bA$, dimension $d$, number of communities $K \geq d$
  \State compute spectral embedding $\hat X_1, \ldots, \hat X_n$ of the graph into $\R^d$
  \State fit a Gaussian mixture model (varying volume, shape, and orientation) with $K$ components
  \Statex \Return cluster centres $\hat \rv_1, \ldots, \hat \rv_K$ and community memberships $\hat Z_1, \ldots, \hat Z_n$
\end{algorithmic}
\end{algorithm}
Where this algorithm differs most significantly from \citet{rohe2011spectral} is in the use of a Gaussian mixture model over $K$-means clustering. In Section~\ref{sec:sbm_illustration}, we show why Theorems~\ref{thrm:GRDPG_ASE_CLT} and~\ref{thrm:GRDPG_LSE_CLT} would recommend this modification, with a pedagogical example.

To accomplish step 2 we have been employing the mclust algorithm \citep{fraley1999mclust}, which has a user-friendly R package. In step 1, either adjacency or Laplacian spectral embedding can be used (see Definition~\ref{def:ASE_and_LSE}). If the latter, the resulting node memberships can be interpreted as alternative estimates of $Z_1, \ldots, Z_n$ but the output cluster centres should be treated as estimating degree-normalised versions of $\rv_1, \ldots, \rv_K$ (see Section~\ref{sec:lse_results}). 

\begin{algorithm}[H]
\caption{Fitting a mixed membership stochastic block model}\label{alg:mmspec}
\begin{algorithmic}[1]
  \Statex \textbf{input} adjacency matrix $\bA$, dimension $d$, number of communities $K = d$
  \State compute adjacency spectral embedding $\hat X_1, \ldots, \hat X_n$ of the graph into $\R^d$
  \State project the data onto the $(d-1)$-dimensional principal hyperplane, to obtain points $\hat Y_1, \ldots, \hat Y_n$, and fit the minimum volume enclosing $K-1$-simplex, with vertices $\hat \rv_1, \ldots, \hat \rv_K$
  \State obtain barycentric coordinates $\hat Y_i = \sum_{k = 1}^K \hat \pi_{il} \hat \rv_k$, for $i = 1, \ldots, n$
  \Statex \Return vertices $\hat \rv_1, \ldots, \hat \rv_K$ of the simplex, and estimated community membership probability vectors $\hat \pi_1, \ldots, \hat \pi_n$
\end{algorithmic}
\label{alg:mves}
\end{algorithm}
This algorithm is unchanged from \citet{rubin2017consistency}, but to prove it is consistent when $\bB$ has negative eigenvalues requires Theorem~\ref{thrm:GRDPG_ASE_ttinf} and Lemma~\ref{lem:indefiniteSpectralBoundGRDPG}. We provide a pedagogical example in Section~\ref{sec:mmsbm_illustration}.

To fit the minimum volume enclosing simplex in step 2, we have been using the algorithm by \citet{lin2016fast} and are grateful to the authors for providing code. Algorithm~\ref{alg:mmspec} can be extended to the case $K \geq d$ by fitting a minimum volume enclosing convex $K$-polytope, but for identifiability one must then assume $\hat \rv_1, \ldots, \hat \rv_K$ are in convex position.

\begin{figure}[htp]
  \centering
  \includegraphics[width=13cm]{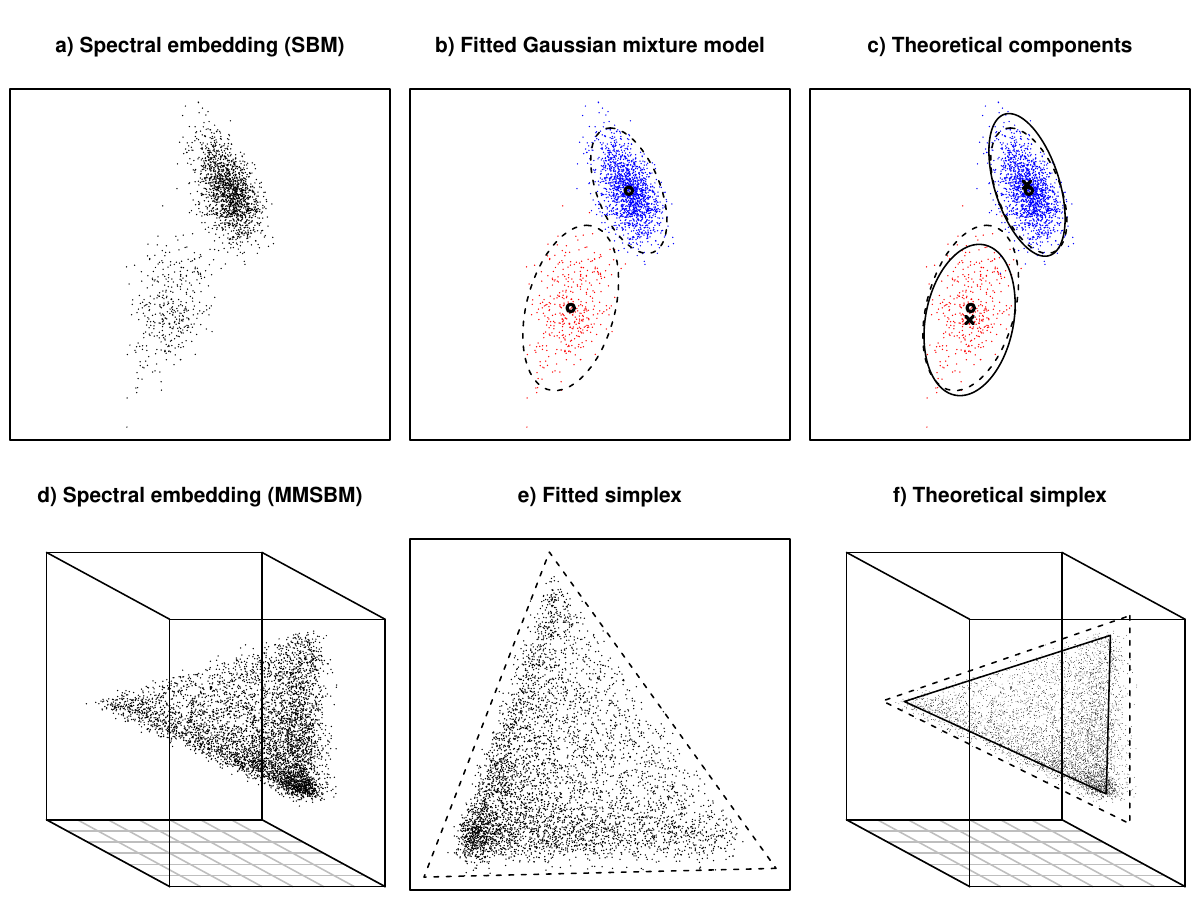}
  \caption{Spectral embedding and analysis of simulated graphs from the standard (SBM --- top) and mixed membership (MMSBM --- bottom) stochastic block models. a) Adjacency spectral embedding into $\R^2$ of a simulated graph with two communities ($n=2000$ nodes); b) two-component Gaussian mixture model (varying volume, shape, and orientation) fit using the R package mclust, coloured by estimated component assignment, with estimated component centres as circles and 95\% Gaussian contours as dashed ellipses; c) two-component Gaussian mixture model predicted by the asymptotic theory, with component centres shown as crosses and 95\% Gaussian contours shown as solid ellipses, overlaid onto the empirical centres (circles) and contours (dashed ellipses) shown in b); d) adjacency spectral embedding into $\R^3$ of a simulated graph with $3$ communities ($n=5000$ nodes); e) minimum volume enclosing simplex (dashed triangle) enclosing the two principal components (points); f) theoretical simplex supporting the latent positions (solid triangle) for comparison with the minimum volume enclosing simplex (dashed triangle). Detailed discussion in Section~\ref{sec:implications}.}
  \label{fig:SBM_MMSBM_superplot}
\end{figure}

\subsection{A two-community stochastic block model}\label{sec:sbm_illustration}
In this section, we show why the central limit theorems (Theorems~\ref{thrm:GRDPG_ASE_CLT}~and~\ref{thrm:GRDPG_LSE_CLT}) would recommend fitting a Gaussian mixture model, rather than using $K$-means, for spectral clustering. To illustrate ideas, we consider a two-community stochastic block model under which every node is independently assigned to the first or second community, with respective probabilities $0.2$ and $0.8$, and the block matrix is
\begin{equation*}
\bB^{(1)} = \left[\begin{array}{cc}0.02 & 0.03 \\ 0.03 & 0.01\end{array}\right],
\end{equation*}
which has one positive and one negative eigenvalue. The two-dimensional adjacency spectral embedding of a simulated graph on $n=2000$ nodes is shown in Figure~\ref{fig:SBM_MMSBM_superplot}a). As we will see, Theorem~\ref{thrm:GRDPG_ASE_CLT} provides a formal sense in which this embedding resembles data from a two-component Gaussian mixture distribution. Figure~\ref{fig:SBM_MMSBM_superplot}b) shows how this model fits the embedding, based on the approximate maximum likelihood parameters found by the mclust algorithm \citep{fraley1999mclust}. The estimated component assignment of each point is indicated by colouring, the empirical component centres are shown as small circles and corresponding empirical 95\% Gaussian contours in dashed lines.

Following the construction of Section~\ref{sec:spesh_sbm}, the graph is a GRDPG with signature $(1,1)$, and its latent positions $X_i \in \R^2$ are i.i.d. from a distribution $F$ which places all its mass on two distinct points, $\mathrm{v}_1$ and $\mathrm{v}_2$, with respective probabilities $0.2$ and $0.8$. The implication of Theorem~\ref{thrm:GRDPG_ASE_CLT} for this example is that, conditional on $Z_i$, the transformed embedding $\bQ \hat X_i$ is approximately distributed as an independent Gaussian vector with centre $\mathrm{v}_{Z_i}$ and covariance $n^{-1/2}\bSigma(\mathrm{v}_{Z_i})$, where $\bQ \in \defO(1,1)$; or, together, the vectors $\bQ \hat X_1, \ldots, \bQ \hat X_n$ approximately follow a two-component Gaussian mixture model.

As a result, the points $\hat X_1, \ldots, \hat X_n$ resemble data from a two-component Gaussian mixture model which have been put through a random, data-dependent, linear transformation $\bQ^{-1}$. Given $\bA$, the cluster assignments, $\rv_1$ and $\rv_2$, we can compute $\bQ$ (see Section~\ref{sec:results}). For the graph that we simulated,
\begin{equation*}
\bQ\approx \left[\begin{array}{cc} 1.05 & -0.32 \\ 0.32 & -1.05\end{array}\right].
\end{equation*}
Figure~\ref{fig:SBM_MMSBM_superplot}c) shows the embedding, the transformed centres $\bQ^{-1}\mathrm{v}_1$, $\bQ^{-1}\mathrm{v}_2$ (crosses), and correspondingly transformed 95\% Gaussian contours (solid ellipses) predicted by Theorem~\ref{thrm:GRDPG_ASE_CLT} for comparison with the empirical versions (circles and dashed ellipses) obtained from the Gaussian mixture model fit.

In practice, we typically observe only $\bA$ and do not have access to $\bQ$, and indeed Algorithm~\ref{alg:spec} is suggesting we fit a Gaussian mixture model to  $\hat X_1, \ldots, \hat X_n$, not  $\bQ \hat X_1, \ldots, \bQ \hat X_n$. Remarkably, by doing the former, we are effectively accomplishing the latter.

This is because, under a Gaussian mixture model, the value of the likelihood is unchanged if, while the component weights are held fixed, the data, component means and covariances are respectively transformed as $X \rightarrow \bM X$, $\mu \rightarrow \bM \mu$, $\Gamma \rightarrow \bM \Gamma \bM^\top$, where $\bM$ is any indefinite orthogonal matrix, for the simple reason that $|\det(\bM)| = 1$. For the maximum likelihood parameters obtained from $\hat X_1, \ldots, \hat X_n$, the maximum \emph{a posteriori} probability assignment of the data indices to mixture components, $\hat Z_1, \ldots, \hat Z_n$, is identical to that which would have been obtained from  $\bQ \hat X_1, \ldots, \bQ \hat X_n$. The cluster centres which would have been obtained from  $\bQ \hat X_1, \ldots, \bQ \hat X_n$ are just $\bQ \hat \rv_1$ and $\bQ \hat \rv_2$, where $\hat \rv_1, \hat \rv_2$ are those actually obtained from $\hat X_1, \ldots, \hat X_n$. The pairs provide identical estimates of $\bB^{(1)}$ through $\hat \bB^{(1)}_{kl} = (\bQ \hat \rv_k)^{\top} \bI_{1,1}(\bQ \hat \rv_l) = \hat \rv_k^{\top} \bI_{1,1} \hat \rv_l$ for $k,l \in \{1,2\}$. In practice regularisation parameters in the clustering method may yield results that are not invariant to indefinite transformation, but such effects should be small for large $n$, especially taken alongside the additional result, in Lemma~\ref{lem:indefiniteSpectralBoundGRDPG}, that the spectral norm of $\bQ$ is almost surely bounded.

As can be seen in Figure~\ref{fig:SBM_MMSBM_superplot}c), the clusters are elliptical rather than spherical, in theory and in practice. For this reason, a clustering algorithm which favours spherical solutions might produce inaccurate results. This is one issue with $K$-means clustering, often said to be implicitly fitting a Gaussian mixture model with equal-volume spherical components \citep{fraley2002model}, and indeed numerical studies \citep{athreya2016limit,tang2016limit} show it has higher misclassification rate on the task of community separation.
Another issue, perhaps more a limitation of our own framework, is that Euclidean distance is not identifiable under the GRDPG, and so the empirical distances $\lVert \hat X_i - \hat X_j \rVert$ are not easily understood through our approach. While our results say enough to prove $K$-means clustering is consistent (a fact already established by other methods), the algorithm is not invariant: the clusterings obtained from  $\hat X_1, \ldots, \hat X_n$ and from $\bQ \hat X_1, \ldots, \bQ \hat X_n$ are different (although they agree asymptotically).

\subsection{A three-community mixed membership stochastic block model} \label{sec:mmsbm_illustration}
In this section, we show how uniform consistency (Theorem~\ref{thrm:GRDPG_ASE_ttinf}) guarantees the consistency of simplex-fitting for mixed membership stochastic block model estimation. To illustrate ideas, we consider a three-community mixed membership stochastic block model under which every node is first independently assigned a 3-dimensional probability vector $\pi_i \sim \text{Dirichlet}(1, 0.5, 0.5)$, for $i = 1, \ldots, n$, and the block matrix is

\begin{equation*}
\bB^{(2)} = \left(\begin{array}{ccc}0.6 & 0.9 & 0.9 \\ 0.9 & 0.6 & 0.9 \\ 0.9 & 0.9 & 0.3 \end{array}\right),
\end{equation*}
which has one positive and two negative eigenvalues. The three-dimensional adjacency spectral embedding of a simulated graph on $n=5000$ nodes is shown Figure~\ref{fig:SBM_MMSBM_superplot}d); this point cloud resembles a `noisy simplex'. Figure~\ref{fig:SBM_MMSBM_superplot}e) shows the minimum volume 2-simplex (dashed lines) enclosing the two principal components of the points. As we will see, Theorem~\ref{thrm:GRDPG_ASE_ttinf} provides a formal sense in which the point cloud becomes denser \emph{and} sharper as $n$ grows, so that this simplex converges. 

Following the construction of Section~\ref{sec:spesh_mmsbm}, the graph is a GRDPG with signature $(1,2)$ and its latent positions $X_i \in \R^3$ are i.i.d. from a distribution $F$ supported on the simplex with vertices $\rv_1, \rv_2, \rv_3 \in \R^3$.

Since $\log^c n/n^{1/2} \rightarrow 0$, Theorem~\ref{thrm:GRDPG_ASE_ttinf} states that the \emph{maximum error} between any $\bQ \hat X_i$ and $X_i$ vanishes, across $i \in \{1, \ldots, n\}$, and this has wide-reaching implications for geometric analysis since, as sets, $\bQ \hat X_1, \ldots, \bQ \hat X_n$ and $X_1, \ldots, X_n$ are asymptotically equal in Hausdorff distance. In particular, for our example, where $\bQ \in \defO(1,2)$, the point set $\bQ \hat X_1, \ldots, \bQ \hat X_n$, like $X_1, \ldots, X_n$, converges in Hausdorff distance to the simplex with vertices  $\rv_1, \rv_2, \rv_3$. This simplex would be consistently estimated by the minimum volume simplex enclosing $\bQ \hat X_1, \ldots, \bQ \hat X_n$,  projected onto their $2$-dimensional principal hyperplane, by the argument presented in \citet{rubin2017consistency}.

We must just verify that the \emph{actual} procedure, i.e., applied to $\hat X_1, \ldots, \hat X_n$, is consistent despite indefinite orthogonal transformation.  The bounded spectral norm of $\bQ^{-1}$ (Lemma~\ref{lem:indefiniteSpectralBoundGRDPG} and discussion) guarantees that
\[\underset{i \in \{1, \ldots, n\}}{\textnormal{max}}\|\hat X_i - \mathbf{Q}^{-1}X_i\| \leq \lVert \bQ^{-1} \rVert \underset{i \in \{1, \ldots, n\}}{\textnormal{max}}\|\bQ \hat X_i - X_i\| = \probO\left(\tfrac{\log^c n}{n^{1/2}}\right),\]
through which one can replace $X_i$ by $\mathbf{Q}^{-1}X_i$ in the argument of \citet{rubin2017consistency}. The minimum volume simplex enclosing $\hat Y_1, \ldots, \hat Y_n$ (the points $\hat X_1, \ldots, \hat X_n$ projected onto their 2-dimensional principal hyperplane, see Algorithm~\ref{alg:mmspec}), is then found to converge to the simplex with vertices $\bQ^{-1} \rv_1, \bQ^{-1} \rv_2, \bQ^{-1} \rv_3$. This simplex is shown in Figure~\ref{fig:SBM_MMSBM_superplot}f) (solid lines). 

As with the stochastic block model, the presence of indefinite orthogonal transformation in the estimated vertices $\hat \rv_1, \hat \rv_2, \hat \rv_3$ is immaterial for estimating $\bB$ through $\hat \bB^{(2)}_{kl} = \hat \rv_k^{\top} \bI_{1,2}\hat \rv_l$ for $k,l \in \{1,2,3\}$. Moreover, the community membership probability vectors $\hat \pi_1, \ldots, \hat \pi_n$, obtained as the barycentric coordinates of $\hat Y_1, \ldots, \hat Y_n$ with respect to  $\hat \rv_1, \hat \rv_2, \hat \rv_3$, are the same as those which would be obtained from $\bQ \hat Y_1, \ldots, \bQ \hat Y_n$ with respect to $\bQ \hat \rv_1, \bQ \hat \rv_2, \bQ \hat \rv_3$.

\section{Real data}\label{sec:real-data}
\subsection{A collection of real-world graphs} \label{sec:graph_survey}
Allowing a principled treatment of negative eigenvalues increases the scope of application of spectral embedding. To gain an impression of the significance of our extension, we conduct a survey of graphs from a variety of application domains. Graphs with about 5,000 nodes were chosen from each of the domain categories of a comprehensive online network repository (\url{networkrepository.com}), selecting the largest if all graphs in a category were smaller, and rejecting the category if all graphs were much larger.

For each of the resulting 24 graphs, an estimated embedding dimension $\hat d$ was obtained using profile likelihood \citep{zhu2006automatic}, and $\hat p$ (respectively, $\hat q$) were estimated as the number of positive (respectively, negative) eigenvalues among the largest $\hat d$ in magnitude. Results are shown in Table~\ref{tab:survey}, and it happens that precisely half of the graphs have $\hat q > 0$. Moreover, the smallest negative eigenvalue often ranks among the largest in magnitude. 

\begin{table} \begin{center}
\begin{tabular}{lrrrrr} \toprule  
  {Graph category} & {nodes} & {edges} & {$\hat{d}$} & {$\hat{p}$} & {$\hat{q}$}\\ \midrule
  animal social & 1,686 & 5,324 & 9 & 9 & 0\\ 
benchmark (BHOSLIB) & 4,000 & 7,425,226 & 2 & 1 & 1\\ 
benchmark (DIMACS 10) & 4,096 & 12,264 & 6 & 6 & 0\\ 
benchmark (DIMACS) & 4,000 & 4,000,268 & 3 & 2 & 1\\ 
biological & 4,413 & 108,818 & 23 & 11 & 12\\ 
brain & 1,781 & 33,641 & 6 & 6 & 0\\ 
cheminformatics & 125 & 282 & 12 & 12 & 0\\ 
collaboration & 4,158 & 13,422 & 1 & 1 & 0\\ 
communication & 1,899 & 61,734 & 54 & 27 & 27\\ 
ecology & 128 & 2,106 & 47 & 47 & 0\\ 
economic & 4,008 & 8,188 & 4 & 2 & 2\\ 
email & 1,133 & 5,451 & 28 & 25 & 3\\ 
infrastructure & 4,941 & 6,594 & 2 & 1 & 1\\ 
interaction & 1,266 & 6,451 & 1 & 1 & 0\\ 
molecular & 5,110 & 10,532 & 16 & 16 & 0\\ 
power & 5,300 & 13,571 & 22 & 22 & 0\\ 
proximity & 410 & 2,765 & 17 & 17 & 0\\ 
retweet & 5,248 & 6,394 & 25 & 20 & 5\\ 
road & 2,642 & 3,303 & 11 & 9 & 2\\ 
router & 2,113 & 6,632 & 13 & 13 & 0\\ 
social (advogato) & 6,551 & 51,332 & 46 & 46 & 0\\ 
social (facebook) & 5,372 & 279,191 & 6 & 5 & 1\\ 
structural mechanics & 5,489 & 143,300 & 12 & 6 & 6\\ 
web & 4,767 & 37,375 & 12 & 10 & 2\\ 
 \bottomrule
\end{tabular}
\end{center}
\caption{A collection of real-world graphs. $\hat{d}$: the estimated dimension, or approximate rank of the adjacency matrix, computed using profile likelihood \citep{zhu2006automatic}; $\hat p$ (respectively, $\hat{q}$): the number of positive (respectively, negative) eigenvalues of the adjacency matrix among the first $\hat d$. The estimate $\hat{q}$ is non-zero for half of these graphs; these negative components can hold important signal, that our generalisation of the random dot product graph allows us to model.} %
\label{tab:survey}
\end{table}

\subsection{Detailed example: link prediction on a computer network}
\label{sec:cyber-data}Cyber-security applications often involve data with a network structure, for example, data relating to computer network traffic \citep{neil2013scan}, the underground economy \citep{li2014identifying}, and the internet-of-things \citep{hp-report15}. In the first example, a concrete reason to seek to develop an accurate network model is to help identify intrusions on the basis of anomalous links \citep{neil2013towards,nah16}.

Figure~\ref{fig:NetFlow_graph_1-5} shows, side by side, graphs of the communications made between computers on the Los Alamos National Laboratory network \citep{kent16}, over a single minute on the left, and five minutes on the right. The graphs were extracted from the ``network flow events'' dataset, by mapping each IP address to a node, and recording an edge if the corresponding two IP addresses are observed to communicate at least once over the specified period.

Neither graph contains a single triangle, i.e., three nodes all connecting to each other. This is a symptom of a broader property, known as heterophily or disassortivity \citep{khor2010concurrency}, that similar nodes are relatively \emph{unlikely} to connect. In computer networks, such behaviour might be expected for a number of reasons, including the common server/client networking model and the physical location of routers (where collection happens) \citep{prd16}. The random dot product graph is unsuited to modelling heterophilic connectivity patterns. For example, any two-community stochastic block model with lower on- than off- diagonal elements is out of scope. The eigenvalues of the adjacency matrix of the 5-minute graph are plotted in Figure~\ref{fig:example_spectrum}, showing an abundance of negative eigenvalues which, again, cannot be modelled by the random dot product graph (apart from as noise).

\begin{figure}
\centering
\includegraphics[width=7cm]{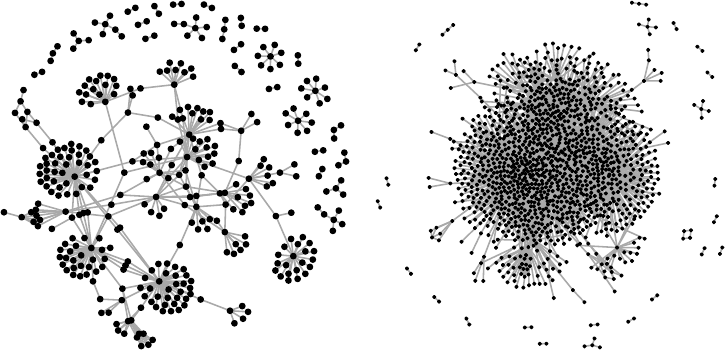}
\caption{Los Alamos National Laboratory computer network. Graphs of the connections made between different computers (IP addresses) over the first minute (left) and first five minutes (right) of the ``network flow events'' dataset \citep{kent16}. Neither graph contains a single triangle, a motif which would be expected in abundance under homophily (`a friend of my friend is my friend'), suggesting the need to relax this modelling assumption.} \label{fig:NetFlow_graph_1-5}
\end{figure}

\begin{figure}
\centering
\includegraphics[width=7cm]{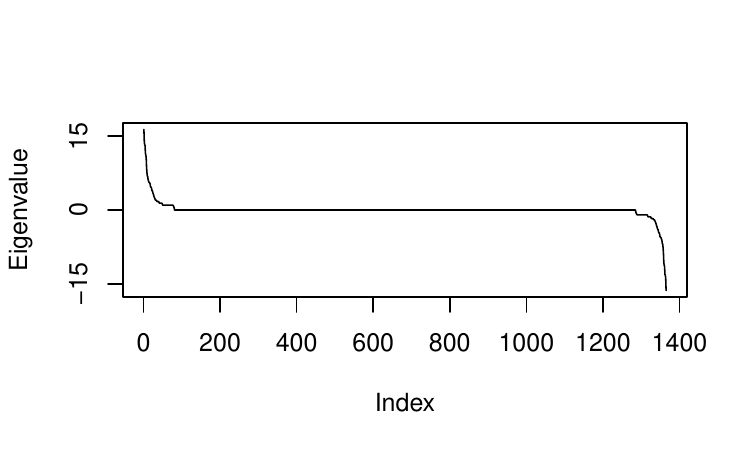}
\caption{Eigenvalues of the adjacency matrix of the five-minute connection graph of computers on the Los Alamos National Laboratory network, showing roughly equal contribution, in magnitude, from the positive and negative eigenvalues.} \label{fig:example_spectrum}
\end{figure}

The modelling improvement offered by the GRDPG over the random dot product graph is now demonstrated empirically, through out-of-sample link prediction. For the observed 5-minute graph, we estimate the GRDPG latent positions via adjacency spectral embedding, as in Definition~\ref{def:ASE_and_LSE}, and the random dot product graph latent positions using an analogous procedure that retains instead only the largest eigenvalues and corresponding eigenvectors. In both cases, we choose $d=10$ (admittedly arbitrarily) as the embedding dimension.

Whereas we have found that certain analysis techniques (e.g. fitting a Gaussian mixture model) are automatically adapted to the GRDPG signature, a concrete estimate of the signature is required here, and $p,q$ are respectively estimated to be the number of positive and negative eigenvalues among the top $d$, in magnitude. 

To compare the models, we then attempt to predict which \emph{new} edges will occur in the next five-minute window, a task known as link prediction, disregarding those involving new nodes. Figure~\ref{fig:BS-OS_ROCs} shows the receiver operating characteristic (ROC) curves for each model, treating the prediction task as a classification problem where the presence or absence of an edge is encoded as an instance of the positive or negative class, respectively, and predicted by thresholding the inner product or indefinite inner product of the relevant pair of estimated latent positions. By presenting estimated classification performance (true positive versus false positive rate) at all possible thresholds (which give different points along the curve), the ROC allows a direct comparison that is independent of the potentially different ranges and scales of the two inner products. For this prediction problem, the GRDPG model is far superior.

\begin{figure}
\centering
\includegraphics[width=7cm]{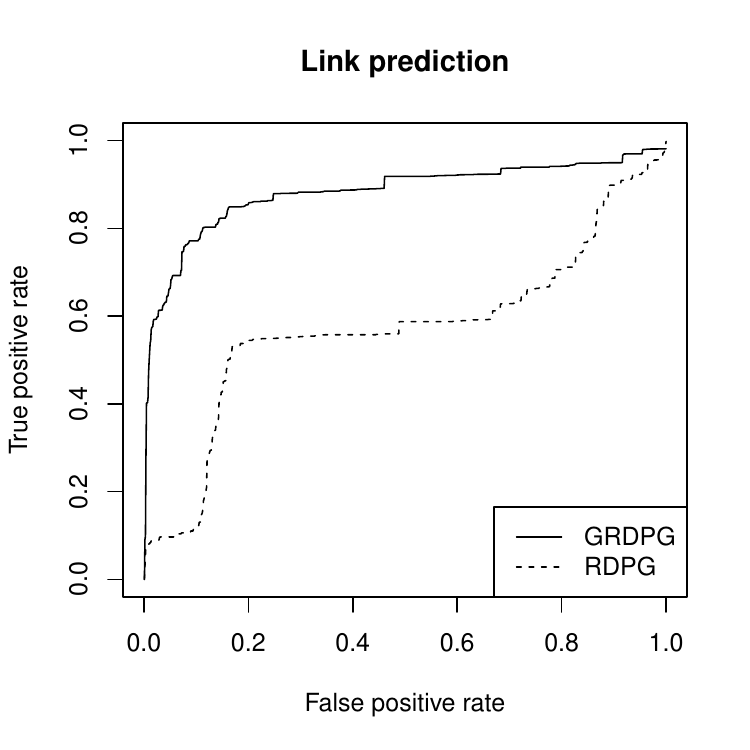}
\caption{Receiver Operating Characteristic curves comparing the random dot product graph (RDPG) and GRDPG at the task of link prediction on the Los Alamos National Laboratory computer network. The nodes are embedded into $\R^{10}$ based on their graph of connections over the first 5-minute window, using either the largest positive (RDPG) or largest magnitude (GRDPG) eigenvalues. The performance of these embeddings is then evaluated for predicting new edges over the next 5-minute window, using the matrix of pairwise inner products (RDPG) or pairwise indefinite inner products as estimated edge probabilities. The ROC curves for each embedding is presented, with the presence or absence of an edge encoded as an instance of the positive or negative class, respectively, and thresholding the estimated edge probabilities to obtain different points along the curve.}\label{fig:BS-OS_ROCs}
\end{figure}

What does the GRDPG model add over the mixed membership or standard stochastic block models? For large real-world networks, the latter models are often too simplistic, whereas the GRDPG model and the statistical investigation thereof, as presented in this paper, provide a more broadly applicable, principled starting point for analyses when low-dimensional latent generative structure is supposed. To illustrate this, we construct the full graph of connections between computers on the Los Alamos National Laboratory network, comprising roughly 12 thousand nodes and one hundred thousand edges. As before, the nodes are spectrally embedded into $\R^{10}$, but these are now visualised in two ways. First, we fit a Gaussian mixture model with ten components, as is consistent with a stochastic block model assumption with $K = d = 10$. Each of the ten panels in Figure~\ref{fig:netflow_clusters} shows the two principal components of one of the inferred clusters in a faithful aspect ratio. Because every communication has an associated port number indicating the type of service being used, for example port 80 corresponds to web activity and port 25 to email, this information can be used to colour the nodes according to their most commonly employed port. The embedding, obtained using only connectivity data, is clearly highly associated with port activity, so that the geometry that can be distinguished appears to be somehow predictive of nodes' behaviour. At the same time the clusters, for the most part, do not appear to follow a Gaussian or finite mixture of Gaussian distributions, as predicted under a stochastic block model. A different view of the data is obtained using t-distributed stochastic neighbour embedding \citep{maaten2008visualizing} in Figure~\ref{fig:tsne}, again showing high association with port activity. Taken together, these views of the data reveal complex structure in low-dimensional pseudo-Euclidean latent space, for which the GRDPG provides a preferrable starting point for statistical analysis to the random dot product graph, the mixed membership, or standard stochastic block model.  

\begin{figure}
\centering
\includegraphics[width=\textwidth]{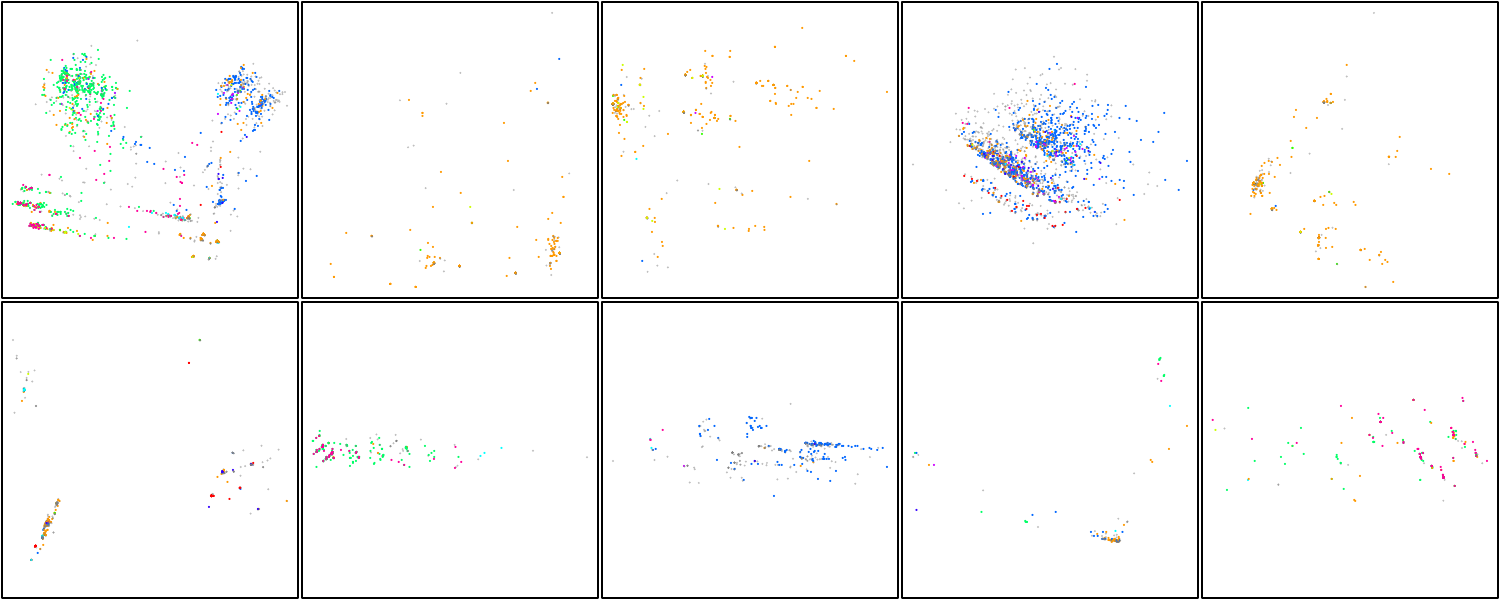}
\caption{Visualisation of the full graph of connections between computers on the Los Alamos National Laboratory network, using adjacency spectral embedding into $\R^{10}$, followed by fitting a Gaussian mixture model with $10$ components. Each of ten clusters obtained is shown in a panel (two principal components) and the colour of each point represents the corresponding node's most commonly employed port (loosely representing the connection purpose, e.g. web, email), showing association with the structure observed in the embedding. The structure is richer than the standard or mixed membership stochastic block models would predict.} \label{fig:netflow_clusters}
\end{figure}

\begin{figure}
\centering
\includegraphics[width=9cm]{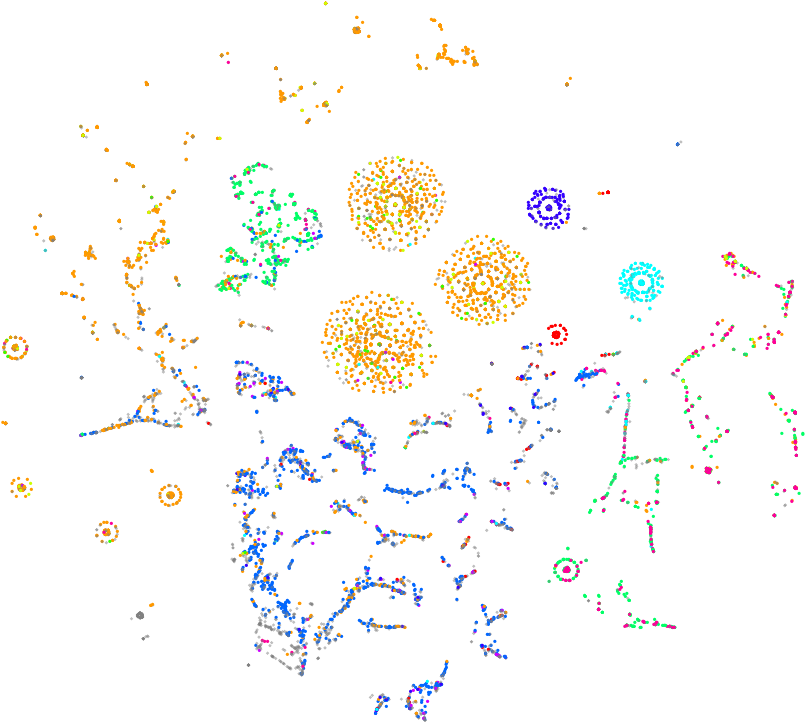}
\caption{Alternative visualisation the full graph of connections between computers on the Los Alamos National Laboratory network, using adjacency spectral embedding into $\R^{10}$ followed by t-distributed stochastic neighbour embedding. The colour of each point represents the corresponding node's most commonly employed port (loosely representing the connection purpose, e.g. web, email), showing association with the structure observed in the embedding.} \label{fig:tsne}
\end{figure}

\section{Conclusion}\label{sec:conclusion}
This paper presents the \emph{generalised random dot product graph}, a latent position model which includes the stochastic block model, its extensions, and the random dot product graph as special cases. The key feature that is added by the generalisation is the possibility of modelling non-homophilic connectivity behaviour, e.g., where `opposites attract'. 

This model provides an appropriate statistical framework for interpreting spectral embedding. This is substantiated in several theoretical results that together show that the vector representations of nodes obtained by spectral embedding provide uniformly consistent latent position estimates with asymptotically Gaussian error. A byproduct of this theory is to add insight and methodological improvements to the estimation of community structure in networks, and practical applications are demonstrated in a cyber-security example.

\bibliographystyle{apalike}
\bibliography{graph_embedding}

\begin{thebibliography}{}

\bibitem[Abbe, 2017]{abbe2017community}
Abbe, E. (2017).
\newblock Community detection and stochastic block models: recent developments.
\newblock {\em The Journal of Machine Learning Research}, 18(1):6446--6531.

\bibitem[Agterberg et~al., 2020]{agterberg2020nonidentifiability}
Agterberg, J., Tang, M., and Priebe, C.~E. (2020).
\newblock On two distinct sources of nonidentifiability in latent position
  random graph models.
\newblock {\em arXiv preprint arXiv:2003.14250}.

\bibitem[Airoldi et~al., 2008]{airoldi2008mixed}
Airoldi, E.~M., Blei, D.~M., Fienberg, S.~E., and Xing, E.~P. (2008).
\newblock Mixed membership stochastic blockmodels.
\newblock {\em Journal of Machine Learning Research}, 9(Sep):1981--2014.

\bibitem[Aldous, 1981]{aldous1981representations}
Aldous, D.~J. (1981).
\newblock Representations for partially exchangeable arrays of random
  variables.
\newblock {\em Journal of Multivariate Analysis}, 11(4):581--598.

\bibitem[Athreya et~al., 2017]{athreya2017statistical}
Athreya, A., Fishkind, D.~E., Tang, M., Priebe, C.~E., Park, Y., Vogelstein,
  J.~T., Levin, K., Lyzinski, V., and Qin, Y. (2017).
\newblock Statistical inference on random dot product graphs: a survey.
\newblock {\em The Journal of Machine Learning Research}, 18(1):8393--8484.

\bibitem[Athreya et~al., 2016]{athreya2016limit}
Athreya, A., Priebe, C.~E., Tang, M., Lyzinski, V., Marchette, D.~J., and
  Sussman, D.~L. (2016).
\newblock A limit theorem for scaled eigenvectors of random dot product graphs.
\newblock {\em Sankhya A}, 78(1):1--18.

\bibitem[Athreya et~al., 2021]{athreya2018estimation}
Athreya, A., Tang, M., Park, Y., and Priebe, C.~E. (2021).
\newblock On estimation and inference in latent structure random graphs.
\newblock {\em Statistical Science}, 36(1):68--88.

\bibitem[Bhatia, 1997]{Bhatia1997}
Bhatia, R. (1997).
\newblock {\em Matrix Analysis}.
\newblock Springer.

\bibitem[Borgatti and Everett, 2000]{borgatti2000models}
Borgatti, S.~P. and Everett, M.~G. (2000).
\newblock Models of core/periphery structures.
\newblock {\em Social networks}, 21(4):375--395.

\bibitem[Cape et~al., 2019a]{cape_biometrika}
Cape, J., Tang, M., and Priebe, C.~E. (2019a).
\newblock Signal-plus-noise matrix models: eigenvector deviations and
  fluctuations.
\newblock {\em Biometrika}, 106(1):243--250.

\bibitem[Cape et~al., 2019b]{cape2017two}
Cape, J., Tang, M., Priebe, C.~E., et~al. (2019b).
\newblock The two-to-infinity norm and singular subspace geometry with
  applications to high-dimensional statistics.
\newblock {\em The Annals of Statistics}, 47(5):2405--2439.

\bibitem[Donath and Hoffman, 1973]{donath1973lower}
Donath, W.~E. and Hoffman, A.~J. (1973).
\newblock Lower bounds for the partitioning of graphs.
\newblock {\em IBM Journal of Research and Development}, 17(5):420--425.

\bibitem[Erd\H{o}s et~al., 2013]{erdos}
Erd\H{o}s, L., Knowles, A., Yau, H.-T., and Yin, J. (2013).
\newblock Spectral statistics of {Erd\H{o}s}-{R}\'{e}nyi' graphs {I}: Local
  semicircle law.
\newblock {\em The Annals of Probability}, 41:2279--2375.

\bibitem[Evans et~al., 2014]{hpdata}
Evans, C., Friedman, J., Karakus, E., and Pandey, J. (2014).
\newblock Potterverse.
\newblock \url{https://github.com/efekarakus/potter-network/}.

\bibitem[Fiedler, 1973]{fiedler1973algebraic}
Fiedler, M. (1973).
\newblock Algebraic connectivity of graphs.
\newblock {\em Czechoslovak mathematical journal}, 23(2):298--305.

\bibitem[Fraley and Raftery, 1999]{fraley1999mclust}
Fraley, C. and Raftery, A.~E. (1999).
\newblock Mclust: Software for model-based cluster analysis.
\newblock {\em Journal of classification}, 16(2):297--306.

\bibitem[Fraley and Raftery, 2002]{fraley2002model}
Fraley, C. and Raftery, A.~E. (2002).
\newblock Model-based clustering, discriminant analysis, and density
  estimation.
\newblock {\em Journal of the American statistical Association},
  97(458):611--631.

\bibitem[Gallier, 2000]{gallier2000curves}
Gallier, J.~H. (2000).
\newblock {\em Curves and surfaces in geometric modeling: theory and
  algorithms}.
\newblock Morgan Kaufmann.

\bibitem[Heard and Rubin-Delanchy, 2016]{nah16}
Heard, N.~A. and Rubin-Delanchy, P. (2016).
\newblock Network-wide anomaly detection via the {D}irichlet process.
\newblock In {\em Proceedings of IEEE workshop on Big Data Analytics for
  Cyber-security Computing}.

\bibitem[{Hewlett Packard Enterprise research study}, 2015]{hp-report15}
{Hewlett Packard Enterprise research study} (2015).
\newblock Internet of things: research study.
\newblock \url{http://h20195.www2.hpe.com/V4/getpdf.aspx/4aa5-4759enw}.

\bibitem[Hoff, 2008]{hoff2008modeling}
Hoff, P. (2008).
\newblock Modeling homophily and stochastic equivalence in symmetric relational
  data.
\newblock In {\em Advances in neural information processing systems},
  volume~20, pages 657--664.

\bibitem[Hoff et~al., 2002]{hoff2002latent}
Hoff, P.~D., Raftery, A.~E., and Handcock, M.~S. (2002).
\newblock Latent space approaches to social network analysis.
\newblock {\em Journal of the American Statistical Association},
  97(460):1090--1098.

\bibitem[Holland et~al., 1983]{holland1983stochastic}
Holland, P.~W., Laskey, K.~B., and Leinhardt, S. (1983).
\newblock Stochastic blockmodels: First steps.
\newblock {\em Social networks}, 5(2):109--137.

\bibitem[Hoover, 1979]{hoover1979relations}
Hoover, D.~N. (1979).
\newblock Relations on probability spaces and arrays of random variables.
\newblock {\em Preprint, Institute for Advanced Study, Princeton, NJ}.

\bibitem[Karrer and Newman, 2011]{karrer2011stochastic}
Karrer, B. and Newman, M.~E. (2011).
\newblock Stochastic blockmodels and community structure in networks.
\newblock {\em Physical Review E}, 83(1):016107.

\bibitem[Kent, 2016]{kent16}
Kent, A.~D. (2016).
\newblock Cybersecurity data sources for dynamic network research.
\newblock In {\em Dynamic Networks and Cyber-Security}. World Scientific.

\bibitem[Khor, 2010]{khor2010concurrency}
Khor, S. (2010).
\newblock Concurrency and network disassortativity.
\newblock {\em Artificial life}, 16(3):225--232.

\bibitem[Labatut and Bost, 2019]{labatut2019extraction}
Labatut, V. and Bost, X. (2019).
\newblock Extraction and analysis of fictional character networks: A survey.
\newblock {\em ACM Computing Surveys (CSUR)}, 52(5):1--40.

\bibitem[Lei, 2018]{lei2018network}
Lei, J. (2018).
\newblock Network representation using graph root distributions.
\newblock {\em arXiv preprint arXiv:1802.09684 (to appear in the Annals of
  Statistics)}.

\bibitem[Lei and Rinaldo, 2015]{lei2015consistency}
Lei, J. and Rinaldo, A. (2015).
\newblock Consistency of spectral clustering in stochastic block models.
\newblock {\em The Annals of Statistics}, 43(1):215--237.

\bibitem[Li and Chen, 2014]{li2014identifying}
Li, W. and Chen, H. (2014).
\newblock Identifying top sellers in underground economy using deep
  learning-based sentiment analysis.
\newblock In {\em Intelligence and Security Informatics Conference (JISIC),
  2014 IEEE Joint}, pages 64--67. IEEE.

\bibitem[Lin et~al., 2016]{lin2016fast}
Lin, C.-H., Chi, C.-Y., Wang, Y.-H., and Chan, T.-H. (2016).
\newblock A fast hyperplane-based minimum-volume enclosing simplex algorithm
  for blind hyperspectral unmixing.
\newblock {\em IEEE Transactions on Signal Processing}, 64(8):1946--1961.

\bibitem[Lloyd, 1982]{lloyd1982least}
Lloyd, S. (1982).
\newblock Least squares quantization in {PCM}.
\newblock {\em IEEE Transactions on Information Theory}, 28(2):129--137.

\bibitem[Lov{\'a}sz, 2012]{lovasz2012large}
Lov{\'a}sz, L. (2012).
\newblock {\em Large networks and graph limits. American Mathematical Society
  Colloquium Publications}, volume~60.
\newblock Amer. Math. Soc. Providence, RI.

\bibitem[Lu and Peng, 2013]{lu2013}
Lu, L. and Peng, X. (2013).
\newblock Spectra of edge-independent random graphs.
\newblock {\em Electronic Journal of Combinatorics}, 20(4).

\bibitem[Lyzinski et~al., 2014]{lyzinski2014perfect}
Lyzinski, V., Sussman, D.~L., Tang, M., Athreya, A., and Priebe, C.~E. (2014).
\newblock Perfect clustering for stochastic blockmodel graphs via adjacency
  spectral embedding.
\newblock {\em Electron. J. Stat.}, 8(2):2905--2922.

\bibitem[Lyzinski et~al., 2017]{lyzinski2017community}
Lyzinski, V., Tang, M., Athreya, A., Park, Y., and Priebe, C.~E. (2017).
\newblock Community detection and classification in hierarchical stochastic
  blockmodels.
\newblock {\em IEEE Transactions on Network Science and Engineering},
  4(1):13--26.

\bibitem[Maaten and Hinton, 2008]{maaten2008visualizing}
Maaten, L. v.~d. and Hinton, G. (2008).
\newblock Visualizing data using t-{SNE}.
\newblock {\em Journal of machine learning research}, 9(Nov):2579--2605.

\bibitem[Mao et~al., 2017]{mao_sarkar}
Mao, X., Sarkar, P., and Chakrabarti, D. (2017).
\newblock Estimating mixed memberships with sharp eigenvector deviations.
\newblock Arxiv preprint at \url{http://arxiv.org/abs/1709.00407}.

\bibitem[Mara et~al., 2020]{mara2020csne}
Mara, A., Mashayekhi, Y., Lijffijt, J., and De~Bie, T. (2020).
\newblock {CSNE}: Conditional signed network embedding.
\newblock {\em arXiv preprint arXiv:2005.10701}.

\bibitem[Neil et~al., 2013a]{neil2013scan}
Neil, J.~C., Hash, C., Brugh, A., Fisk, M., and Storlie, C.~B. (2013a).
\newblock Scan statistics for the online detection of locally anomalous
  subgraphs.
\newblock {\em Technometrics}, 55(4):403--414.

\bibitem[Neil et~al., 2013b]{neil2013towards}
Neil, J.~C., Uphoff, B., Hash, C., and Storlie, C. (2013b).
\newblock Towards improved detection of attackers in computer networks: New
  edges, fast updating, and host agents.
\newblock In {\em 6th International Symposium on Resilient Control Systems
  (ISRCS)}, pages 218--224. IEEE.

\bibitem[Newman, 2018]{newman2018networks}
Newman, M. (2018).
\newblock {\em Networks: an introduction}.
\newblock Oxford university press.

\bibitem[Nickel, 2006]{nickel06}
Nickel, C. (2006).
\newblock {\em Random Dot Product Graphs: A Model for Social Networks}.
\newblock PhD thesis, Johns Hopkins University.

\bibitem[Qin and Rohe, 2013]{qin2013regularized}
Qin, T. and Rohe, K. (2013).
\newblock Regularized spectral clustering under the degree-corrected stochastic
  blockmodel.
\newblock {\em Adv. Neural. Inf. Process. Syst.}, 26:3120--3128.

\bibitem[Rohe et~al., 2011]{rohe2011spectral}
Rohe, K., Chatterjee, S., and Yu, B. (2011).
\newblock Spectral clustering and the high-dimensional stochastic blockmodel.
\newblock {\em The Annals of Statistics}, 39(4):1878--1915.

\bibitem[Rohe et~al., 2018]{rohe2018note}
Rohe, K., Tao, J., Han, X., and Binkiewicz, N. (2018).
\newblock A note on quickly sampling a sparse matrix with low rank expectation.
\newblock {\em The Journal of Machine Learning Research}, 19(1):3040--3052.

\bibitem[Rubin-Delanchy, 2020]{rubin2020manifold}
Rubin-Delanchy, P. (2020).
\newblock Manifold structure in graph embeddings.
\newblock In {\em Proceedings of the Thirty-fourth Conference on Neural
  Information Processing Systems}.

\bibitem[Rubin-Delanchy et~al., 2016]{prd16}
Rubin-Delanchy, P., Adams, N.~M., and Heard, N.~A. (2016).
\newblock Disassortivity of computer networks.
\newblock In {\em Proceedings of IEEE workshop on Big Data Analytics for
  Cyber-security Computing}.

\bibitem[Rubin-Delanchy et~al., 2017]{rubin2017consistency}
Rubin-Delanchy, P., Priebe, C.~E., and Tang, M. (2017).
\newblock Consistency of adjacency spectral embedding for the mixed membership
  stochastic blockmodel.
\newblock {\em arXiv preprint arXiv:1705.04518}.

\bibitem[Sanna~Passino et~al., 2020]{passino2020spectral}
Sanna~Passino, F., Heard, N.~A., and Rubin-Delanchy, P. (2020).
\newblock Spectral clustering on spherical coordinates under the
  degree-corrected stochastic blockmodel.
\newblock {\em arXiv preprint arXiv:2011.04558}.

\bibitem[Sarkar et~al., 2015]{sarkar2015role}
Sarkar, P., Bickel, P.~J., et~al. (2015).
\newblock Role of normalization in spectral clustering for stochastic
  blockmodels.
\newblock {\em The Annals of Statistics}, 43(3):962--990.

\bibitem[Shi and Malik, 2000]{shi2000normalized}
Shi, J. and Malik, J. (2000).
\newblock Normalized cuts and image segmentation.
\newblock {\em IEEE Transactions on pattern analysis and machine intelligence},
  22(8):888--905.

\bibitem[Steinhaus, 1956]{steinhaus1956division}
Steinhaus, H. (1956).
\newblock Sur la division des corp mat\'eriels en parties.
\newblock {\em Bulletin L'Acad\'emie Polonaise des Sciences}, 1(804):801.

\bibitem[Sussman et~al., 2012]{sussman2012consistent}
Sussman, D.~L., Tang, M., Fishkind, D.~E., and Priebe, C.~E. (2012).
\newblock A consistent adjacency spectral embedding for stochastic blockmodel
  graphs.
\newblock {\em Journal of the American Statistical Association},
  107(499):1119--1128.

\bibitem[Tang et~al., 2017]{tang_efficient}
Tang, M., Cape, J., and Priebe, C.~E. (2017).
\newblock Asymptotically efficient estimators for stochastic blockmodels: the
  naive {MLE}, the rank-constrained {MLE}, and the spectral.
\newblock arXiv preprint at \url{http://arxiv.org/abs/1710.10936}.

\bibitem[Tang et~al., 2018]{tang2016limit}
Tang, M., Priebe, C.~E., et~al. (2018).
\newblock Limit theorems for eigenvectors of the normalized laplacian for
  random graphs.
\newblock {\em The Annals of Statistics}, 46(5):2360--2415.

\bibitem[Trosset et~al., 2020]{trosset2020learning}
Trosset, M.~W., Gao, M., Tang, M., and Priebe, C.~E. (2020).
\newblock Learning 1-dimensional submanifolds for subsequent inference on
  random dot product graphs.
\newblock {\em arXiv preprint arXiv:2004.07348}.

\bibitem[Von~Luxburg, 2007]{von2007tutorial}
Von~Luxburg, U. (2007).
\newblock A tutorial on spectral clustering.
\newblock {\em Statistics and Computing}, 17(4):395--416.

\bibitem[Whiteley et~al., 2021]{whiteley2021matrix}
Whiteley, N., Gray, A., and Rubin-Delanchy, P. (2021).
\newblock Matrix factorisation and the interpretation of geodesic distance.
\newblock {\em arXiv preprint arXiv:2106.01260}.

\bibitem[Young and Scheinerman, 2007]{young2007random}
Young, S.~J. and Scheinerman, E.~R. (2007).
\newblock Random dot product graph models for social networks.
\newblock In {\em International Workshop on Algorithms and Models for the
  Web-Graph}, pages 138--149. Springer.

\bibitem[Yu et~al., 2015]{yu2015useful}
Yu, Y., Wang, T., and Samworth, R.~J. (2015).
\newblock A useful variant of the {D}avis--{K}ahan theorem for statisticians.
\newblock {\em Biometrika}, 102(2):315--323.

\bibitem[Zhu and Ghodsi, 2006]{zhu2006automatic}
Zhu, M. and Ghodsi, A. (2006).
\newblock Automatic dimensionality selection from the scree plot via the use of
  profile likelihood.
\newblock {\em Computational Statistics \& Data Analysis}, 51(2):918--930.

\end{thebibliography}

\newpage
\appendix
\section{Indefinite orthogonal transformations in practice}
In the next two sections, we respond to the possible argument that concerns about distortion by an indefinite orthogonal $\bQ$ arise only as an artifact of the GRDPG formalism; or, at least, that such concerns are of little relevance if sole interest in spectral embedding is to allow inference for stochastic block models.
\subsection{Distortion under the stochastic block model} \label{sec:distortion1}%

A practical experiment that reveals the presence of indefinite orthogonal transformation implicit in spectral embedding is to simulate graphs from two stochastic block models that differ only in their community proportions. As in Section \ref{sec:sbm_illustration} we will use $K=2$ and the probability matrix $\bB^{(1)}$, with community proportions now set to $(0.5, 0.5)$ and $(0.05, 0.95)$ respectively, as opposed to $(0.2, 0.8)$. Resulting spectral embeddings on $n = 4000$ nodes are shown in the left-hand panel of Figure~\ref{fig:unidentifiability_in_practice}, in orange and purple respectively. While each exhibits two clusters, there is very little overlap between the orange and purple point clouds. This discrepancy is almost entirely due to distortion by indefinite orthogonal transformation. Since in simulation $\bQ$ can be identified, by inversion the indefinite orthogonal transformation that takes us from the purple to the orange point cloud can be found, and those conforming point clouds are shown in the right-hand panel. The centres of corresponding clusters are now in much closer agreement, and remaining discrepancy is down to statistical error. One cannot make sense of the geometric relationship between the two point clouds without the notion of an indefinite orthogonal transformation.

If an indefinite orthogonal transformation is applied to $\hat \bX$ before clustering using $K$-means (with Euclidean distance), a different partition of the points is obtained. One might have hoped that the spectral decomposition would produce an embedding somehow optimally configured for clustering using Euclidean $K$-means, but there is no obvious statistical argument to prefer $\hat \bX$. For example, within-class variance was previously used to compare spectral embeddings under the stochastic block model \citep{sarkar2015role}, but both the empirical (cluster assignment estimated) and oracle (cluster assignment known) within-class variances are minimised for a different, non-degenerate configuration. There are no such concerns with Gaussian mixture modelling, which is invariant. The arguments here support, but go beyond, previous analyses showing the suboptimality of $K$-means clustering versus Gaussian mixture modelling under the non-negative-definite stochastic block model \citep{tang2016limit}, since in that special case $K$-means clustering is at least invariant.

\begin{figure}
  \centering
  \includegraphics[width=\textwidth]{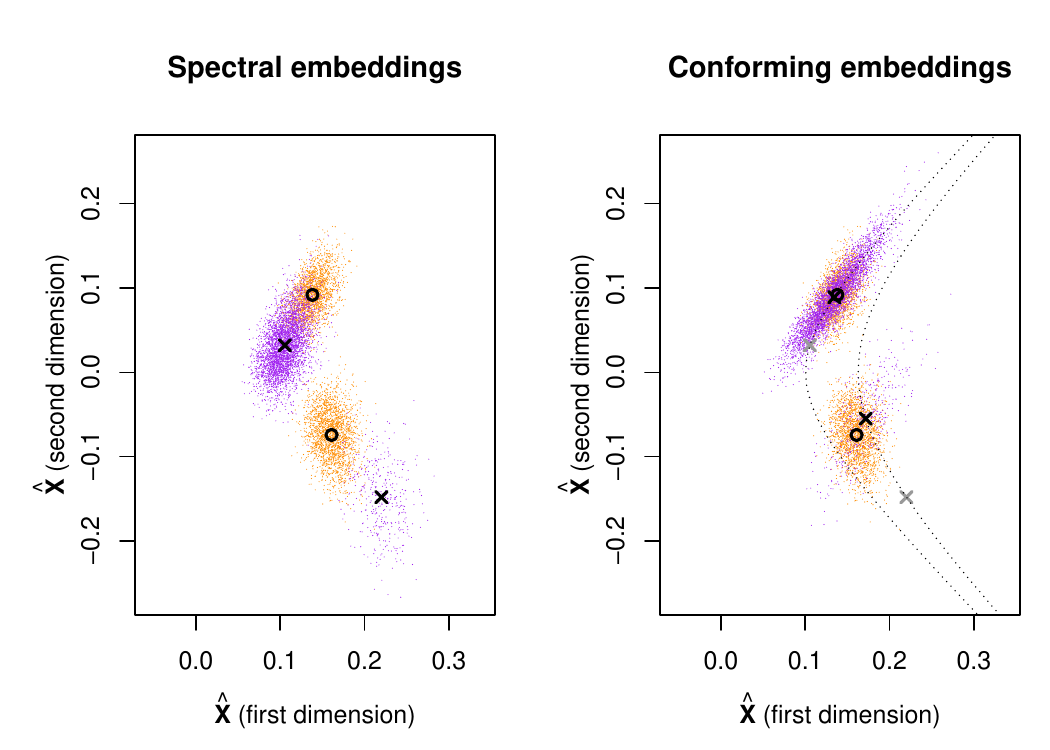}
  \caption{A practical manifestation of $\bQ$. Two graphs are generated from a two-community stochastic block model with block matrix $\bB^{(1)}$, $n=4000$ nodes, and respective community proportions $(0.5, 0.5)$ (orange) and $(0.05, 0.95)$ (purple). Left: adjacency spectral embedding into $\R^2$, with circles (respectively crosses) indicating the cluster centres of the orange (respectively purple) point cloud. Right: the purple point cloud is re-configured to align with the orange point cloud, by reverting the indefinite orthogonal transformation associated with the former, and then applying that corresponding to the latter. The dotted lines show the orbits along which the cluster centres were moved, and the grey crosses their original position. The cluster centres of both point clouds (black circles and crosses) are now close.}
  \label{fig:unidentifiability_in_practice}
\end{figure}

\subsection{Distortion under the degree-corrected stochastic block model} \label{sec:distortion2}%
\begin{figure}[htp]
  \centering
  \includegraphics[width=\textwidth]{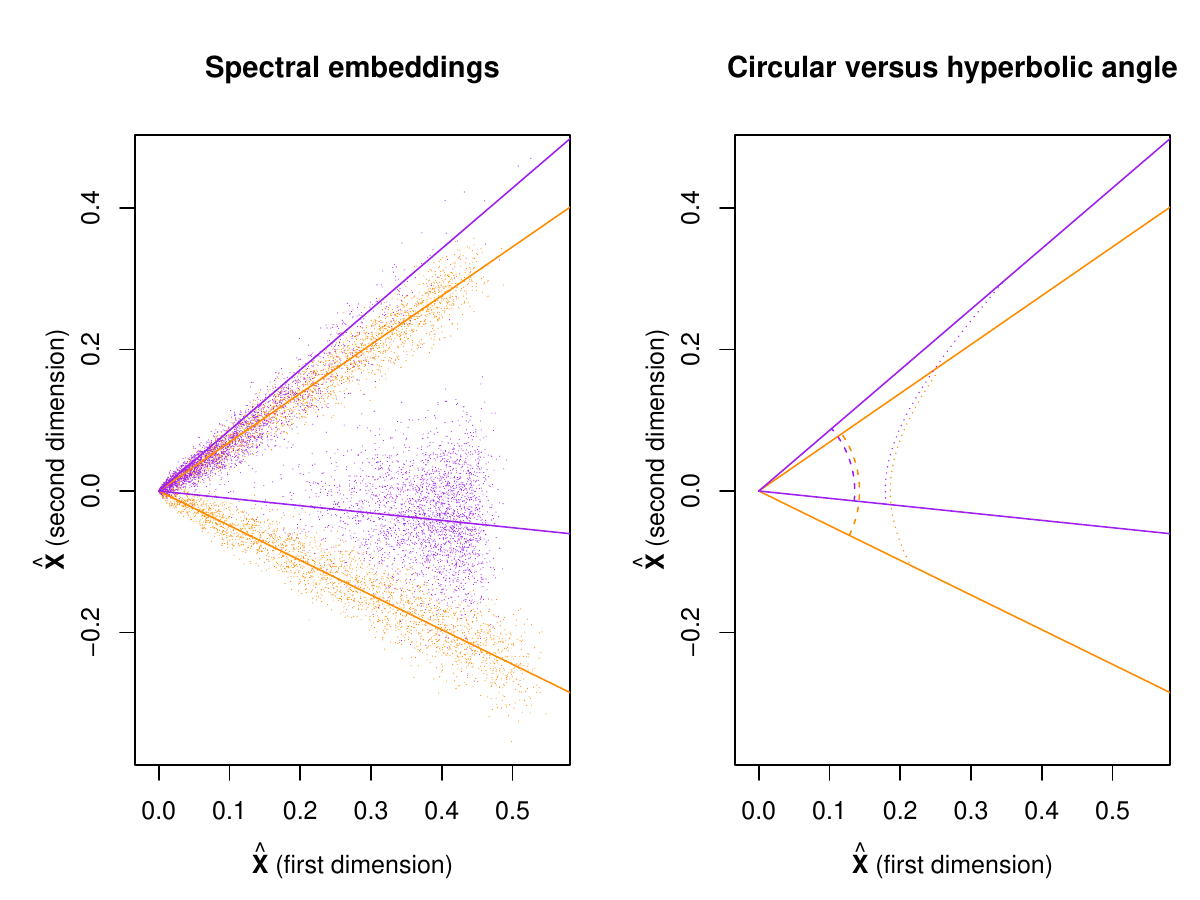}
  \caption{Distortion of angle under the degree-corrected stochastic block model. Two graphs are generated from a two-community degree-corrected stochastic block model with probability matrix $\bB^{(1)}$, $n=5000$ nodes, community proportions $(0.5, 0.5)$, and uniform (orange) versus differing Beta-distributed (purple) weights. Left: adjacency spectral embedding into $\R^2$. In each case the point cloud is a noisy observation of two rays through the origin, shown with orange and purple lines respectively. Right: Euclidean (dashed) and hyperbolic (dotted) angles between the rays. The two hyperbolic angles are identical but the two Euclidean angles differ.}
  \label{fig:c_v_h_angle}
\end{figure}

Two graphs on $n=5000$ nodes are now simulated from a two-community degree-corrected stochastic block model, with block matrix $\bB^{(1)}$, and community proportions $(0.5, 0.5)$, changing only the degree distributions. In the first case we set $w_i \overset{i.i.d}{\sim} \text{uniform}[0,1]$, whereas in the second $w_i \overset{ind}{\sim} \text{Beta}(1,5)$ when $Z_i = 1$ and $w_i \overset{ind}{\sim} \text{Beta}(5,1)$ when $Z_i = 2$. Resulting spectral embeddings into $\R^2$ are shown in the left-hand panel of Figure~\ref{fig:c_v_h_angle}, in orange and purple respectively. Because $\bB^{(1)}$ has one positive and one negative eigenvalue, our theory predicts that each point cloud should live close to the union of two rays through the origin whose joint configuration is predicted \emph{only} up to indefinite orthogonal transformation in $\defO(1,1)$. In particular the hyperbolic angle between the two rays
\[\acosh(\rv_1^\top \bI_{1,1} \rv_2),\]
is a population quantity that subject to regularity conditions (e.g. as given in Theorem~\ref{thrm:GRDPG_ASE_ttinf}) can be estimated consistently and is not dependent on the node weights. This would make a natural measure of distance between the communities and, in general, when $p = 1$ and $ q \geq 1$ the point cloud $\hat X_i/(\hat X^{\top}_i \bI_{1,q} \hat X_i)$ equipped with distance $d(x,y) = \acosh(x^\top \bI_{1,q} y)$ is an embedding into hyperbolic space that accounts for degree heterogeneity. Accordingly, the hyperbolic angle between the two orange rays and that between the two purple rays are equal in the right panel of Figure~\ref{fig:c_v_h_angle}. On the other hand the Euclidean (or ordinary) angle $\acos(u_1^\top u_2)$, where $u_1$, $u_2$ are unit-norm vectors on each ray, is visibly different between the two point clouds.

\section{Uniqueness}
There are a number of reasonable alternative latent position models which, broadly described, assign the nodes to elements $X_1, \ldots, X_n$ of a set $\mathcal{X}$ and, with this assignment held fixed, set
\[A_{ij}\overset{ind}{\sim} \text{Bernoulli}\left\{f(X_i, X_j)\right\},\]
for $i < j$, where $f : \mathcal{X}^2 \rightarrow [0,1]$ is some symmetric function. For example, \citet{hoff2002latent} considered the choice $f(x,y) = \text{logistic}\left(\alpha-\lVert x - y\rVert\right)$. What is special about the GRDPG?

One argument for considering the GRDPG is that it provides essentially the only way of faithfully reproducing mixtures of connectivity probability profiles as convex combinations in latent space. This idea is now made formal.

\begin{property}[Reproducing mixtures of connectivity probability profiles] \label{prop:propertyp}
Suppose that $\mathcal{X}$ is a convex subset of a real vector space, and that $S$ is a subset of $\mathcal{X}$ whose convex hull is $\mathcal{X}$. We say that a symmetric function  $f : \mathcal{X}^2 \rightarrow [0,1]$ reproduces mixtures of connectivity probability profiles from $S$ if, whenever $x = \sum_r \alpha_r u_r$, where $u_r \in S$, $0 \leq \alpha_r\leq 1$ and  $\sum \alpha_r = 1$, we have
 \[f(x,y) = \sum_{r} \alpha_r f(u_r,y),\]
for any $y$ in $\mathcal{X}$.
\end{property}
This property helps interpretation of latent space. For example, suppose $X_1, \ldots, X_4 \in S$, and $X_1 = 1/2 X_2 + 1/2 X_3$. In a latent position model where $f$ satisfies the above, we can either think of $\bA_{14}$ as being directly generated through $\bA_{14}\overset{ind}{\sim} \text{Bernoulli}\left\{f(X_1, X_4)\right\}$, or by first flipping a coin, and generating an edge with probability $f(X_2, X_4)$ if it comes up heads, or with probability $f(X_3, X_4)$ otherwise. 

In choosing a latent position model to represent the mixed membership stochastic block model, it would be natural to restrict attention to kernels satisfying Property~\ref{prop:propertyp}, since they allow the simplex representation illustrated in Figure~\ref{fig:mmsbm_illustration}, with vertices $S = \{\rv_1, \ldots, \rv_K\}$ representing communities, and latent positions within it reflecting the nodes' community membership preferences.

We now find that in finite dimension, any such choice amounts to a GRDPG model in at most one extra dimension: 
\begin{theorem}\label{thm:representation}
Suppose $\mathcal{X}$ is a subset of $\R^{l}$, for some $l \in \N$. The function $f$ reproduces mixtures of connectivity probability profiles if and only if there exist integers $p \geq 1$, $q \geq 0$, $d=p+q\leq l+1$, a matrix $\bT \in \R^{d \times l}$, and a vector $\nu \in \R^{d}$ so that $f(x,y) = (\bT x + \nu)^\top \indefI (\bT y + \nu)$, for all $x,y \in \mathcal{X}$.
\end{theorem}
The mixed membership stochastic block model is an example where this additional dimension is required: in Figure~\ref{fig:mmsbm_illustration} the model is represented as a GRDPG model in $d=3$ dimensions, but the latent positions live on a $2$-dimensional subset. 

\subsection{Proof of Theorem \ref{thm:representation}}
Let $\text{aff}(C)$ denote the \emph{affine hull} of a set $C \subseteq \R^d$,
\[\text{aff}(C) = \left\{\sum_{i=1}^n \alpha_i u_i; n \in \N, u_i \in C, \alpha_i \in \R, \sum_{i=1}^n \alpha_i = 1 \right\}.\]
We say that a function $g: \R^d \times \R^d \rightarrow \R$ is a bi-affine form if it is an affine function when either argument is fixed, i.e., $g\{\lambda x_1 + (1-\lambda) x_2, y\} = \lambda g(x_1,y) + (1-\lambda) g(x_2,y)$ and $g\{x, \lambda y_1 + (1-\lambda) y_2\} = \lambda g(x,y_1) + (1-\lambda) g(x,y_2)$, for any $x,y, x_1, x_2, y_1, y_2 \in \R^d$,  $\lambda \in \R$. We say that a function $h: \R^d \times \R^d \rightarrow \R$ is a bilinear form if it is bi-affine and $h(x,y) = 0$ if either argument is zero.

The proof of Theorem~\ref{thm:representation} is a direct consequence of the following two lemmas. 
\begin{lemma}\label{lem:bi-affine}
Suppose $\mathcal{X}$ is a convex subset of $\R^{l}$, for some $l \in \N$. Then $f$ reproduces mixtures of connectivity probability profiles on $S$ if and only if it can be extended to a symmetric bi-affine form $g : \text{aff}(\mathcal{X}) \times \text{aff}(\mathcal{X}) \rightarrow \R$.
\end{lemma}
\begin{lemma}\label{lem:bi-linear}
Suppose $g : \text{aff}(\mathcal{X}) \times \text{aff}(\mathcal{X}) \rightarrow \R$ is a bi-affine form. Let $\ell = \dim\{\aff(\mathcal{X})\} \leq l$. Then there exist a matrix $\bR \in \R^{(\ell+1) \times l}$, a vector $\mu \in \R^{\ell+1}$, and a bilinear form $h: \R^{(\ell+1)} \times \R^{(\ell+1)} \rightarrow \R$ such that $g(x,y) = h(\bR x + \mu, \bR y + \mu)$, for all $x,y \in \aff(\mathcal{X})$.
\end{lemma}
As is well-known, because $h$ is a symmetric bilinear form on a finite-dimensional real vector space, it can be written $h(x,y) = x^\top \bJ y$ where $\bJ \in \R^{(\ell+1)\times(\ell+1)}$ is a symmetric matrix. Write $\bJ = \bV_d \bS_d \bV_d^\top$ where $\bV_d \in \R^{(\ell+1) \times d}$ has orthonormal columns, $\bS_d \in \R^{d \times d}$ is diagonal and has $p \geq 0$ positive followed by $q \geq 0$ negative eigenvalues on its diagonal, and $d = p+q = \rank(\bJ)$. Next, define $\bM = \bV_d |\bS_d|^{1/2}$. Then,
 \begin{align*}f(x,y) &= g(x,y) = h(\bR x + \mu, \bR y + \mu)\\ &= \left\{\bM (\bR x + \mu)\right\}^\top \indefI \left\{\bM (\bR y + \mu)\right\} = (\bT x + \nu)^\top \indefI (\bT y + \nu),\end{align*}
where $\bT = \bM \bR$ and $\nu = \bM \mu$. Since $f(x,x) \geq 0$ on $\mathcal{X} \times \mathcal{X}$, we must have $p > 0$ unless $f$ is uniformly zero over $\mathcal{X} \times \mathcal{X}$.

\begin{proof}[Proof of Lemma~\ref{lem:bi-affine}]
The ``if'' part of the proof is straightforward. Here, we prove the ``only if''. By definition, any $x,y \in \text{aff}(\mathcal{X}) = \text{aff}(S)$ can be written $x =\sum \alpha_r u_r$, $y = \sum \beta_r v_r$ where  $u_r, v_r \in S$, $\alpha_r,\beta_r \in \R$, and $\sum \alpha_r = \sum \beta_r = 1$. For any such $x,y$, we define $g(x,y) = \sum_{r,s} \alpha_r \beta_s f(u_r, v_s)$.

Suppose that $\sum \alpha_r u_r = \sum \gamma_r t_r, \sum \beta_r v_r = \sum \delta_r w_r$ where $t_r, w_r \in S$, $\gamma_r,\delta_r, \in \R$, and $\sum \gamma_r = \sum \delta_r = 1$. Rearrange the first equality to $\sum \alpha'_r u'_r = \sum \gamma'_r t'_r$ by moving any $\alpha_r u_r$ term where $\alpha_r < 0$ to the right --- so that the corresponding new coefficient is $\alpha'_s = - \alpha_r$, for some $s$ --- and any $\gamma_r t_r$ term where $\gamma_r < 0$ to the left, so that the corresponding new coefficient is $\gamma'_s = -\gamma_r$, for some $s$. Both linear combinations now involve only non-negative scalars. Furthermore, $\sum \alpha_r = \sum \gamma_r$ ($=1$) implies $\sum \alpha'_r = \sum \gamma'_r = c$, for some $c \geq 0$.

Then,  $\sum (\alpha'_r/c) u'_r = \sum (\gamma'_r/c) t'_r$ are two convex combinations, therefore, 
\begin{align*}
\sum (\alpha'_r/c) f(u_r',v) &= f\left\{\sum (\alpha'_r/c) u'_r, v\right\}
= f\left\{\sum (\gamma'_r/c) t'_r,v\right\}
= \sum (\gamma'_r/c) f(t'_r,v),
\end{align*}
for any $v \in S$, so that $\sum \alpha_r f(u_r, v) = \sum \gamma_r f(t_r, v)$. Therefore, 
\begin{align*}
\sum_{r,s} \alpha_r \beta_s f(u_r, v_s) &= \sum_s \beta_s \left\{\sum_r \gamma_r f(t_r, v_s)\right\}\\
& = \sum_r \gamma_r \left\{\sum_s \beta_s f(v_s, t_r)\right\} = \sum_{r,s} \gamma_r \delta_s f(t_r, w_s),
\end{align*}
so that $g$ is well-defined. The function $g$ is symmetric and it is also clear that $g\{\lambda x_1 + (1-\lambda) x_2, y\} = \lambda g(x_1,y) + (1-\lambda) g(x_2,y)$ for any $\lambda \in \R$, making it bi-affine by symmetry.
\end{proof}

The proof technique now used is known as the homogenisation trick in geometry \citep{gallier2000curves}.

\begin{proof}[Proof of Lemma~\ref{lem:bi-linear}]
Let $x_0, x_1, \ldots, x_{\ell} \in \R^{l}$ be an affine basis of $\aff(\mathcal{X})$. Then there exists an affine transformation $x \rightarrow \bR x + \mu$, mapping $x_0$ to $z_0 = (0, \ldots, 0, 1) \in \R^{\ell+1}$, $x_1$ to $z_1 = (1, 0, \ldots, 0, 1)$, and so forth, finally mapping $x_\ell$ to $z_{\ell} = (0, \dots, 0, 1, 1)$, where $\bR \in \R^{(\ell+1) \times l}$ and $\mu \in \R^{\ell+1}$. The vectors $z_0, \ldots, z_{\ell}$ form a basis of $\R^{\ell+1}$, so that if we set $h(z_i, z_j) = g(x_i, x_j)$ for $0 \leq i,j \leq \ell$, then the value $h$ is well-defined over $\R^{(\ell+1)} \times \R^{(\ell+1)}$ by bilinearity and basis expansion. Since any $x,y \in \aff(\mathcal{X})$ can be written $x =\sum_{r=0}^\ell \alpha_r x_r$, $y = \sum_{r=0}^\ell \beta_r x_r$ where $\alpha_r,\beta_r \in \R$, and $\sum \alpha_r = \sum \beta_r = 1$, we have 
\begin{align*}
g(x,y) &= \sum_{r,s} \alpha_r \beta_s g(x_r, x_s) = \sum_{r,s} \alpha_r \beta_s h(z_r, z_s)\\
& = h\left(\sum \alpha_r z_r, \sum \beta_r z_r\right) = h(\bR x + \mu, \bR y + \mu).
\end{align*}
\end{proof}

\section{Proof of Theorems~\ref{thrm:GRDPG_ASE_ttinf}~and~\ref{thrm:GRDPG_ASE_CLT}}
Broadly speaking, extending prior results on adjacency spectral embedding from the random dot product graph to the GRDPG requires new methods of analysis, that together represent the main technical contribution of this paper (mainly Theorems~\ref{thrm:GRDPG_ASE_ttinf} and \ref{thrm:GRDPG_ASE_CLT}). Further extending results to the case of Laplacian spectral embedding, while mathematically involved, follows \emph{mutatis mutandis} the machinery developed in \citet{tang2016limit}. Analogous Laplacian-based results (Theorems~\ref{thrm:GRDPG_LSE_ttinf} and \ref{thrm:GRDPG_LSE_CLT}) are therefore stated without proof.
\subsection{Preliminaries}
This proof synthesizes and adapts both the proof architecture and machinery developed in the papers \cite{tang_efficient,cape2017two,cape_biometrika}. We invoke the probabilistic concentration phenomena for GRDPGs presented in Lemma~7 of \cite{tang2016limit} as well as an eigenvector matrix series decomposition concisely analyzed in \cite{cape_biometrika}. This proof is not, however, a trivial corollary of earlier results, for it requires additional technical considerations and insight.

We recall the setting of Theorems~\ref{thrm:GRDPG_ASE_ttinf}~and~\ref{thrm:GRDPG_ASE_CLT} wherein the rows of $\mathbf{X}$ are independent replicates of the random vector $\rho_n^{1/2} \xi$, $\xi \sim F$, and for $i < j$ the $ij$-th entries of $\mathbf{A}$ are independent Bernoulli random variables with mean $X_i^{\top} \mathbf{I}_{p,q} X_j$. Here we shall allow self-loops for mathematical convenience since (dis)allowing self-loops is immaterial with respect to the asymptotic theory we pursue. 

Now for $\bP = \bX\indefI\bX^{\top}$, denote the low-rank spectral decomposition of $\bP$ by $\bP = \bU\bS\bU^{\top}$, where $\bU \in\mathbb{O}_{n,d}$ and $\bS\in\R^{d \times d}$. Write $\bU\equiv[\bU_{(+)}|\bU_{(-)}]$ with $\bU_{(+)} \in\mathbb{O}_{n,p}$ and  $\bU_{(-)}\in\mathbb{O}_{n,q}$ to indicate the orthonormal eigenvectors corresponding to the $p$ positive and $q$ negative non-zero eigenvalues of $\bP$, written in block-diagonal matrix form as $\bS = \bS_{(+)} \bigoplus \bS_{(-)} \in \R^{(p+q)\times(p+q)} \equiv \R^{d \times d}$. Denote the full spectral decomposition of $\bA$ by $\bA = \hat{\bU}\hat{\bS}\hat{\bU}^{\top} + \hat{\bU}_{\perp}\hat{\bS}_{\perp}\hat{\bU}_{\perp}^{\top}$, where $\hat{\bU} \in \defO_{n,d}$ denotes the matrix of leading (orthonormal) eigenvectors of $\bA$ and $\hat{\bS}\in\R^{d \times d}$ denotes the diagonal matrix containing the $d$ largest-in-magnitude eigenvalues of $\bA$ arranged in decreasing order. Here, the matrix $\hat{\bU}\hat{\bS}\hat{\bU}^{\top}$ corresponds to the best canonical rank $d$ representation of $\bA$. Also above, write $\hat{\bU} \equiv[\hat{\bU}_{(+)}|\hat{\bU}_{(-)}]$ such that the columns of $\hat{\bU}_{(+)}$ and $\hat{\bU}_{(-)}$ consist of orthonormal eigenvectors corresponding to the largest $p$ positive and $q$ negative non-zero eigenvalues of $\bA$, respectively.

We remark at the onset that for the GRDPG model, asymptotically almost surely $\|\bU\|_{\ttinf} = O(n^{-1/2})$ and $|\bS_{ii}|$, $|\hat{\bS}_{ii}| = \Theta((n\rho_{n}))$ for each $i = 1, \dots, d$. Simultaneously, $\|\bA-\bP\|=\probO((n\rho_{n})^{1/2})$ (regarding the latter, see for example~\cite{lu2013,lei2015consistency}).

Before going into the details of the proof, we first show that 
$\bU^{\top}\hat{\bU}$ is sufficiently close to an orthogonal matrix $\mathbf{W}_*$ with block diagonal structure that is simultaneously an element of $\mathbb{O}(p,q)$. To this end, the matrix $\bU^{\top}\hat{\bU}$ can be written in block form as
\begin{equation}
	\bU^{\top}\hat{\bU}
	= \Bigl[
	\begin{smallmatrix}
		\bU_{(+)}^{\top}\hat{\bU}_{(+)} &\vline& \bU_{(+)}^{\top}\hat{\bU}_{(-)} \\ \hline
		\bU_{(-)}^{\top}\hat{\bU}_{(+)} &\vline& \bU_{(-)}^{\top}\hat{\bU}_{(-)}
	\end{smallmatrix}
	\Bigr] \in \R^{d \times d},
\end{equation}
where
$\bU_{(+)}^{\top}\hat{\bU}_{(+)} \in \R^{p \times p}$,
$\bU_{(+)}^{\top}\hat{\bU}_{(-)} \in \R^{p \times q}$,
$\bU_{(-)}^{\top}\hat{\bU}_{(+)} \in \R^{q \times p}$,
and
$\bU_{(-)}^{\top}\hat{\bU}_{(-)} \in \R^{q \times q}$. 

Write the singular value decomposition of $\bU_{(+)}^{\top}\hat{\bU}_{(+)}\in\R^{p \times p}$ as
$\bU_{(+)}^{\top}\hat{\bU}_{(+)}
\equiv \bW_{(+),1}\bSigma_{(+)}\bW_{(+),2}^{\top}$, and define the orthogonal matrix $\bW_{(+)}^{\star}:= \bW_{(+),1}\bW_{(+),2}^{\top}\in\defO_{p}$. Similarly, let $\bW_{(-)}^{\star}\in\defO_{q}$ denote the orthogonal (product) matrix corresponding to $\bU_{(-)}^{\top}\hat{\bU}_{(-)}$. Now let $\bW_{\star}$ denote the structured orthogonal matrix
\begin{equation} \label{eq:wstar}
	\bW_{\star}
	=	\Bigl[
	\begin{smallmatrix}
		\bW_{(+)}^{\star} &\vline& \mathbf{0} \\ \hline
		\mathbf{0} &\vline& \bW_{(-)}^{\star}
	\end{smallmatrix}
	\Bigr] \in \defO_{d}.
\end{equation}
Observe that $\bW_{\star}\indefI\bW_{\star}^{\top}=\indefI$, hence simultaneously $\bW_{\star} \in \indefO$.
Via the triangle inequality, the spectral norm quantity $\|\bU^{\top}\hat{\bU} - \bW_{\star}\|$ is bounded above by four times the largest spectral norm of its blocks. The main diagonal blocks can be analyzed in a straightforward manner via canonical angles and satisfy 
\begin{equation}
	\|\bU_{(+)}^{\top}\hat{\bU}_{(+)} - \bW_{(+)}^{\star}\|,
	\|\bU_{(-)}^{\top}\hat{\bU}_{(-)} - \bW_{(-)}^{\star}\|
	= \probO((n\rho_{n})^{-1}).
\end{equation}
More specifically, let $\sigma_1, \sigma_2, \dots, \sigma_p$ be the singular values of $\bU_{(+)}^{\top}\hat{\bU}_{(+)}$. Then $\sigma_i = \cos(\theta_i)$ where $\theta_i$ are the principal angles between the subspaces spanned by $\bU_{(+)}$ and $\hat{\bU}_{(+)}$. The definition of $\bW_{(+)}$ implies
$$ \|\bU_{(+)}^{\top}\hat{\bU}_{(+)} - \bW_{(+)}^{\star} \|_{F} = \|\bSigma_{(+)} - \mathbf{I}\|_{F} = \Bigl(\sum_{i=1}^{p} (1 - \sigma_i)^2 \Bigr)^{1/2} \leq \sum_{i=1}^{p}(1 - \sigma_i^2) = \|\bU_{(+)} \bU_{(+)}^{\top} - \hat{\bU}_{(+)} \hat{\bU}_{(+)}^{\top}\|_{F}^{2}.$$
By the Davis-Kahan $\sin \Theta$ theorem (see e.g.~Section~VII.3 of \cite{Bhatia1997} or \cite{yu2015useful}), we have
$$\|\bU_{(+)}^{\top}\hat{\bU}_{(+)} - \bW_{(+)}^{\star} \|_{F} \leq \|\bU_{(+)} \bU_{(+)}^{\top} - \hat{\bU}_{(+)} \hat{\bU}_{(+)}^{\top}\|_{F}^{2} \leq \frac{C \|\mathbf{A} - \mathbf{P}\|^{2}}{\lambda_{p}(\mathbf{P})^{2}} = O_{\mathbb{P}}((n \rho_n)^{-1}),$$
where $\lambda_{p}(\mathbf{P})$ is the smallest positive eigenvalue of $\mathbf{P}$. The bound $\|\bU_{(-)}^{\top}\hat{\bU}_{(-)} - \bW_{(-)}^{\star}\| = O_{\mathbb{P}}((n \rho_n)^{-1})$ is derived similarly.

We now bound the quantities $\|\bU_{(+)}^{\top}\hat{\bU}_{(-)}\|$. Let $\bm{u}_{i,(+)}$ and $\hat{\bm{u}}_{j,(-)}$ be arbitrary columns of $\bU_{(+)}$ and $\hat{\bU}_{(-)}$, respectively. Note that the $ij$-th entry of $\bU_{(+)}^{\top}\hat{\bU}_{(-)}$ is $(\bm{u}_{i,(+)})^{\top} \hat{\bm{u}}_{j,(-)}$ and that $\lambda_{i,(+)} (\bm{u}_{i,(+)})^{\top} \hat{\bm{u}}_{j,(-)} = (\bm{u}_{i,(+)})^{\top} \mathbf{P} \hat{\bm{u}}_{j,(-)}$, $\hat{\lambda}_{j,(-)} (\bm{u}_{i,(+)})^{\top} \hat{\bm{u}}_{j,(-)} = (\bm{u}_{i,(+)})^{\top} \mathbf{A} \hat{\bm{u}}_{j,(-)}$ where $\lambda_{i,(+)}$ (resp.~$\hat{\lambda}_{j,(-)}$) is the $i$-th (resp.~$j$-th) largest in modulus positive eigenvalue (resp.~negative eigenvalue) of $\mathbf{P}$ (resp.~$\mathbf{A}$). We therefore have
\begin{equation*}
\begin{split}
(\bm{u}_i^{(+)})^{\top} \hat{\bm{u}}_j^{(-)} &= (\hat{\lambda}_{j,(-)} - \lambda_{i,(+)})^{-1} (\bm{u}_{i,(+)})^{\top} (\mathbf{A} - \mathbf{P}) \hat{\bm{u}}_{j,(-)} \\ &= (\hat{\lambda}_{j,(-)} - \lambda_{i,(+)})^{-1} (\bm{u}_{i,(+)})^{\top} (\mathbf{A} - \mathbf{P}) \bU_{(-)} \bU_{(-)}^{\top} \hat{\bm{u}}_{j,(-)} \\ &+ (\hat{\lambda}_{j,(-)} - \lambda_{i,(+)})^{-1} (\bm{u}_{i,(+)})^{\top} (\mathbf{A} - \mathbf{P}) (\mathbf{I} - \bU_{(-)} \bU_{(-)}^{\top}) \hat{\bm{u}}_{j,(-)}.
\end{split}
\end{equation*}
The term $(\bm{u}_{i,(+)})^{\top} (\mathbf{A} - \mathbf{P}) \bU_{(-)}$ is a vector in $\R^{q}$, and conditional on $\mathbf{P}$, each element of $(\bm{u}_{i,(+)})^{\top} (\mathbf{A} - \mathbf{P}) \bU_{(-)}$ can be written as a sum of independent random variables. Hence, by Hoeffding's inequality, $\|(\bm{u}_{i,(+)})^{\top} (\mathbf{A} - \mathbf{P}) \bU_{(-)}\| = O_{\mathbb{P}}(\log{n})$. Furthermore, by the Davis-Kahan theorem, $\|(\mathbf{I} - \bU_{(-)} \bU_{(-)}^{\top}) \hat{\bm{u}}_{j,(-)}\| = O_{\mathbb{P}}((n \rho_n)^{-1/2})$. We therefore have
\begin{gather} 
\label{eq:gather1}
\|(\hat{\lambda}_{j,(-)} - \lambda_{i,(+)})^{-1} (\bm{u}_{i,(+)})^{\top} (\mathbf{A} - \mathbf{P}) \bU_{(-)} \bU_{(-)}^{\top} \hat{\bm{u}}_{j,(-)}\| = O_{\mathbb{P}}((n \rho_n^{-1}) \log{n}); \\
\label{eq:gather2}
\|(\hat{\lambda}_{j,(-)} - \lambda_{i,(+)})^{-1} (\bm{u}_{i,(+)})^{\top} (\mathbf{A} - \mathbf{P}) (\mathbf{I} - \bU_{(-)} \bU_{(-)}^{\top}) \hat{\bm{u}}_{j,(-)}\| = O_{\mathbb{P}}((n \rho_n)^{-1}).
\end{gather}
Equations~\eqref{eq:gather1}~and~\eqref{eq:gather2} together imply
\begin{align}
	\label{eq:offDiagUtopUhat}
	\|\bU_{(+)}^{\top}\hat{\bU}_{(-)}\|
	= \probO((n\rho_{n})^{-1}(\log n)),
\end{align}
thus $\|\bU^{\top}\hat{\bU} - \bW_{\star}\| = \probO((n\rho_{n})^{-1}(\log n))$.

\subsection{Proof details}
We now proceed with the proof of Theorem~\ref{thrm:GRDPG_ASE_ttinf}~and~Theorem~\ref{thrm:GRDPG_ASE_CLT}. 
The matrix relation $\hat{\bU}\hat{\bS}=\bA\hat{\bU}=(\bP+(\bA-\bP))\hat{\bU}$ yields the matrix equation $\hat{\bU}\hat{\bS}-(\bA-\bP)\hat{\bU} = \bP\hat{\bU}$. The spectra of $\hat{\bS}$ and $\bA-\bP$ are disjoint asymptotically almost surely, so $\hat{\bU}$ can be written as a matrix series of the form (see e.g. Theorem~VII.2.1 and Theorem~VII.2.2 of \cite{Bhatia1997})
\begin{equation}
	\label{eq:Uhat_expansion}
	\hat{\bU}
	= \sum_{k=0}^{\infty}(\bA-\bP)^{k}\bP\hat{\bU}\hat{\bS}^{-(k+1)}
	= \sum_{k=0}^{\infty}(\bA-\bP)^{k}\bU\bS\bU^{\top}\hat{\bU}\hat{\bS}^{-(k+1)}.
\end{equation}
By scaling the matrix $\hat{\bU}$ by $|\hat{\bS}|^{1/2}$, observing that $\hat{\bS}=\indefI|\hat{\bS}|$, and applying a well-thought-out ``plus zero" trick, we arrive at the decomposition 
\begin{align*}
	\hat{\bU}|\hat{\bS}|^{1/2}
	&= \sum_{k=0}^{\infty}(\bA-\bP)^{k}\bU\bS\bU^{\top}\hat{\bU}\indefI^{k+1}|\hat{\bS}|^{-k-1/2}\\
	&= \sum_{k=0}^{\infty}(\bA-\bP)^{k}\bU\indefI|\bS|^{-k+1/2}\bW_{\star}\indefI^{k+1}\\
	& \hspace{1em} + \sum_{k=0}^{\infty}(\bA-\bP)^{k}\bU\indefI|\bS|^{-k+1/2}(\bU^{\top}\hat{\bU}-\bW_{\star})\indefI^{k+1}\\
	& \hspace{1em} + \sum_{k=0}^{\infty}(\bA-\bP)^{k}\bU\bS(\bU^{\top}\hat{\bU}\indefI^{k+1}|\hat{\bS}|^{-k-1/2} - |\bS|^{-k-1/2}\bU^{\top}\hat{\bU}\indefI^{k+1})\\
	&:= \bV_{1} + \bV_{2} + \bV_{3}.
\end{align*}
\subsubsection{The matrix $\bV_{1}$}
Diagonal matrices commute, as do the matrices $\indefI$ and $\bW_{\star}$, so $\bV_{1}$ can be written as
\begin{equation}
	\bV_{1}
	\equiv \sum_{k=0}^{\infty}(\bA-\bP)^{k}\bU|\bS|^{-k+1/2}\bW_{\star}\indefI^{k+2}
	= \bU|\bS|^{1/2}\bW_{\star}
	+ (\bA-\bP)\bU|\bS|^{-1/2}\bW_{\star}\indefI + \bR_{\bV_{1}},
\end{equation}
where $\bR_{\bV_{1}} = \sum_{k=2}^{\infty}(\bA-\bP)^{k}\bU|\bS|^{-k+1/2}\bW_{\star}\indefI^{k+2}$.
We now use the following slight restatement of Lemma~7.10 from \cite{erdos}. This result was also noted in \cite{mao_sarkar}.
\begin{lemma}
\label{lem:erdos}
Assume the setting and notations in Theorem~\ref{thrm:GRDPG_ASE_ttinf}. Let $\bm{u}_j$ be the $j$-th column of $\mathbf{U}$ for $j = 1,2, \dots, d$. Then there exists a (universal) constant $c  > 1$ such that for all $k \leq \log{n}$ 
$$ \|(\mathbf{A} - \mathbf{P})^{k} \mathbf{U} \|_{2 \to \infty} \leq
d^{1/2} \max_{j \in [d]} \|(\mathbf{A} - \mathbf{P})^{k} \bm{u}_j\|_{\infty} = O_{\mathbb{P}}\Bigl(\frac{(n \rho_n)^{k/2} \log^{kc}(n)}{n^{1/2}}\Bigr).$$
\end{lemma}
Thus, for $c > 0$ as above,
\begin{align*}
	\|\bR_{\bV_{1}}\|_{\ttinf}
	&\le \sum_{k=2}^{\log n}\|(\bA-\bP)^{k}\bU\|_{\ttinf}\||\bS|^{-1}\|^{k-1/2}
	+ \sum_{k > \log n}\|\bA-\bP\|^{k}\||\bS|^{-1}\|^{k-1/2}\\
	&= \sum_{k=2}^{\log n}\probO\left(\tfrac{d^{1/2}(\log n)^{k c}}{n^{1/2}(n\rho_{n})^{k/2-1/2}}\right) + \sum_{k > \log n}\probO\left((n\rho_{n})^{-k/2+1/2}\right)\\
	&= \probO\left(\tfrac{d^{1/2}(\log n)^{2 c}}{n^{1/2}(n\rho_{n})^{1/2}}\right)
	+ \probO\left((n\rho_{n})^{-(\log n)/2}\right)\\
	&= \probO\left(\tfrac{d^{1/2}(\log n)^{2 c}}{n^{1/2}(n\rho_{n})^{1/2}}\right)
	+ \probO\left(\tfrac{1}{n^{1/2}(n\rho_{n})^{1/2}}\right).
\end{align*}
Moving forward, we set forth to make precise the sense in which
\begin{align*}
	\hat{\bU}|\hat{\bS}|^{1/2}
	&= \bU|\bS|^{1/2}\bW_{\star} + (\bA-\bP)\bU|\bS|^{-1/2}\bW_{\star}\indefI
	+ \bR_{\bV_{1}} + \bV_{2} + \bV_{3}\\
	&\approx \bU|\bS|^{1/2}\bW_{\star} + (\bA-\bP)\bU|\bS|^{-1/2}\bW_{\star}\indefI.
\end{align*}
\subsubsection{The matrix $\bV_{2}$}
For the matrix $\bV_{2}:=\sum_{k=0}^{\infty}(\bA-\bP)^{k}\bU\indefI|\bS|^{-k+1/2}(\bU^{\top}\hat{\bU}-\bW_{\star})\indefI^{k+1}$, it is sufficient to observe that by properties of two-to-infinity norm and the bounds established above,
\begin{align*}
	\|\bV_{2}\|_{\ttinf}
	&\le \|\bU\|_{\ttinf}\||\bS|^{1/2}\|\|\bU^{\top}\hat{\bU}-\bW_{\star}\| + \|(\bA-\bP)\bU\|_{\ttinf}\||\bS|^{-1}\|^{1/2}\|\bU^{\top}\hat{\bU}-\bW_{\star}\| + \|\bR_{\bV_{2}}\|_{\ttinf}\\
	&= \probO\left(\tfrac{\log n}{n^{1/2}(n\rho_{n})^{1/2}}\right)
	+ \probO\left(\tfrac{d^{1/2}(\log n)^{c+1}}{n^{1/2}(n\rho_{n})}\right)
	+ \probo(\|\bR_{\bV_{1}}\|_{\ttinf})\\
	&= \probO\left(\tfrac{d^{1/2}(\log n)^{2 c}}{n^{1/2}(n\rho_{n})^{1/2}}\right).
\end{align*}
In the above, we write that the random variable $Y \in \R$ is $\probo(f(n))$ if for any positive constant $c>0$ and any $\epsilon > 0$ there exists an $n_{0}$ such that for all $n \ge n_{0}$, $|Y| \le \epsilon f(n)$ with probability at least $1-n^{-c}$. 

\subsubsection{The matrix $\bV_{3}$}
The matrix $\bV_{3}$ is given by
$\bV_{3} = \sum_{k=0}^{\infty}(\bA-\bP)^{k}\bU\bS(\bU^{\top}\hat{\bU}\indefI^{k+1}|\hat{\bS}|^{-k-1/2} - |\bS|^{-k-1/2}\bU^{\top}\hat{\bU}\indefI^{k+1})$.
For each $k=0,1,2,\dots$, define the matrix $\bM_{k} := (\bU^{\top}\hat{\bU}\indefI^{k+1}|\hat{\bS}|^{-k-1/2} - |\bS|^{-k-1/2}\bU^{\top}\hat{\bU}\indefI^{k+1})$. Entry $ij$ of the matrix $\bM_{k}$ can be written as
\begin{align*}
	(\bM_{k})_{ij}
	&= \langle u_{i}, \hat{u}_{j} \rangle (\indefI^{k+1})_{jj} \left[|\hat{\bS}_{jj}|^{-k-1/2}
	- |\bS_{ii}|^{-k-1/2}\right]\\
	&= \langle u_{i}, \hat{u}_{j} \rangle (\indefI^{k+1})_{jj} \left[|\hat{\bS}_{jj}|^{-2k-1}
	- |\bS_{ii}|^{-2k-1}\right] \left[|\hat{\bS}_{jj}|^{-k-1/2}
	+ |\bS_{ii}|^{-k-1/2}\right]^{-1}\\
	&= - \langle u_{i}, \hat{u}_{j} \rangle (\indefI^{k+1})_{jj} \left[|\hat{\bS}_{jj}|^{2k+1}
	- |\bS_{ii}|^{2k+1}\right]
	\left[|\hat{\bS}_{jj}|^{-k-1/2}
	+ |\bS_{ii}|^{-k-1/2}\right]^{-1}
	|\hat{\bS}_{jj}|^{-2k-1}
	|\bS_{ii}|^{-2k-1}\\
	&= - \langle u_{i}, \hat{u}_{j} \rangle (\indefI^{k+1})_{jj} \left[|\hat{\bS}_{jj}|
	- |\bS_{ii}|\right]
	\left[
	\sum_{l=0}^{2k}|\hat{\bS}_{jj}|^{l}|\bS_{ii}|^{2k-l}
	\right]
	\left[|\hat{\bS}_{jj}|^{-k-1/2}
	+ |\bS_{ii}|^{-k-1/2}\right]^{-1}
	|\hat{\bS}_{jj}|^{-2k-1}
	|\bS_{ii}|^{-2k-1}.
\end{align*}
For each $k$, further define a matrix $\bH_{k} \in \R^{d \times d}$ entrywise as
\begin{equation}
	\label{eq:matrixH_k_entry}
	(\bH_{k})_{ij} :=
	\left[
	\sum_{l=0}^{2k}|\hat{\bS}_{jj}|^{l}|\bS_{ii}|^{2k-l}
	\right]
	\left[
	|\hat{\bS}_{jj}|^{-k-1/2}
	+ |\bS_{ii}|^{-k-1/2}
	\right]^{-1}
	|\hat{\bS}_{jj}|^{-2k-1}
	|\bS_{ii}|^{-2k-1},
\end{equation}
where it follows that
\begin{equation}
	\label{eq:matrixH_k_entrybound}
	(\bH_{k})_{ij} = \probO((k+1)(n\rho_{n})^{-k-3/2}).
\end{equation}
Letting $\circ$ denote the Hadamard matrix product, we arrive at the decomposition
\begin{equation}
	\label{eq:matrixM_k}
	\bM_{k} = -\bH_{k} \circ (\bU^{\top}\hat{\bU}\indefI^{k+1}|\hat{\bS}| - |\bS|\bU^{\top}\hat{\bU}\indefI^{k+1}).
\end{equation}
The matrices $\bU^{\top}\hat{\bU}$ and $\indefI$ approximately commute. More precisely,
\begin{equation}
	\left(
	\indefI\bU^{\top}\hat{\bU} - \bU^{\top}\hat{\bU}\indefI
	\right)
	=
	\Bigl[
	\begin{smallmatrix}
		\mathbf{0} &\vline& 2\bU_{(+)}^{\top}\hat{\bU}_{(-)} \\ \hline
		-2\bU_{(-)}^{\top}\hat{\bU}_{(+)} &\vline& \mathbf{0}
	\end{smallmatrix}
	\Bigr] \in \R{^{d \times d}},
\end{equation}
so by Eq.~(\ref{eq:offDiagUtopUhat}), the spectral norm of this matrix difference behaves as $\probO((n\rho_{n})^{-1}(\log n))$. This approximate commutativity is important in light of further decomposing the matrix $\bM_{k}$ as
\begin{align*}
	\label{eq:matrixM_k_expanded}
	\bM_{k}
	&= 
	-\bH_{k} \circ (\bU^{\top}\hat{\bU}\indefI^{k+1}|\hat{\bS}| - |\bS|\bU^{\top}\hat{\bU}\indefI^{k+1})\\
	&=
	-\bH_{k} \circ \left(
	\left(\bU^{\top}\hat{\bU}\indefI|\hat{\bS}|\indefI^{k} - |\bS|\indefI\bU^{\top}\hat{\bU}\indefI^{k}\right)
	+ |\bS|\left(\indefI\bU^{\top}\hat{\bU}-\bU^{\top}\hat{\bU}\indefI\right)\indefI^{k}\right)\\
	&=
	-\bH_{k} \circ \left(
	\left(
	\bU^{\top}\hat{\bU}\hat{\bS} - \bS\bU^{\top}\hat{\bU}
	\right)\indefI^{k}
	+ |\bS|\left(\indefI\bU^{\top}\hat{\bU}-\bU^{\top}\hat{\bU}\indefI\right)\indefI^{k}\right).
\end{align*}
We note that $\bU^{\top}\hat{\bU}\hat{\bS} - \bS\bU^{\top}\hat{\bU} = \mathbf{U}^{\top} (\mathbf{A} - \mathbf{P}) \hat{\bU} =
\mathbf{U}^{\top} (\mathbf{A} - \mathbf{P}) \bU \bU^{\top} \hat{\bU} + \mathbf{U}^{\top} (\mathbf{A} - \mathbf{P}) (\mathbf{I} - \bU \bU^{\top}) \hat{\bU}$ and once again, by Hoeffding's inequality and the Davis-Kahan theorem, we have
$\|\bU^{\top}\hat{\bU}\hat{\bS} - \bS\bU^{\top}\hat{\bU}\| = \probO(\log n)$, so $\bM_{k}$ can be bounded as
\begin{align*}
	\|\bM_{k}\|
	&\le \|\bH_{k}\|
	\|(
	\bU^{\top}\hat{\bU}\hat{\bS} - \bS\bU^{\top}\hat{\bU}
	)\indefI^{k}
	+ |\bS|(\indefI\bU^{\top}\hat{\bU}-\bU^{\top}\hat{\bU}\indefI)\indefI^{k}\|\\
	&\le d\|\bH_{k}\|_{\rmMax}
	\left[
	\|\bU^{\top}\hat{\bU}\hat{\bS} - \bS\bU^{\top}\hat{\bU}\| + \||\bS|\|\|\indefI\bU^{\top}\hat{\bU}-\bU^{\top}\hat{\bU}\indefI\|
	\right]\\
	&= \probO(d(k+1)(n\rho_{n})^{-k-3/2})
	\left[
	\probO(\log n)
	+
	\probO((n\rho_{n}) \times (n\rho_{n})^{-1}(\log n))
	\right]\\
	&= \probO(d(k+1)(\log n)(n\rho_{n})^{-k-3/2}).
\end{align*}
Hence, for the matrix $\bV_{3}$,
\begin{align*}
	\|\bV_{3}\|_{\ttinf}
	&\le \|\bU\bS\bM_{0}\|_{\ttinf} + \|(\bA-\bP)\bU\bS\bM_{1}\|_{\ttinf} + \sum_{k=2}^{\infty}\|(\bA-\bP)^{k}\bU\bS\bM_{k}\|_{\ttinf},
\end{align*}
where
\begin{align*}
	\|\bU\bS\bM_{0}\|_{\ttinf}
	&\le \|\bU\|_{\ttinf}\|\bS\|\bM_{0}\|
	= \probO\left(\tfrac{d(\log n)}{n^{1/2}(n\rho_{n})^{1/2}}\right),\\
	\|(\bA-\bP)\bU\bS\bM_{1}\|_{\ttinf}
	&\le \|(\bA-\bP)\bU\|_{\ttinf}\|\bS\|\|\bM_{1}\|
	= \probO\left(\tfrac{d^{3/2}(\log n)^{c+1}}{n^{1/2}(n\rho_{n})}\right),
\end{align*}
and
\begin{align*}
	\sum_{k=2}^{\infty}\|(\bA-\bP)^{k}\bU\bS\bM_{k}\|_{\ttinf}
	&\le \sum_{k=2}^{\log n}\|(\bA-\bP)^{k}\bU\bS\bM_{k}\|_{\ttinf}
	+ \sum_{k> \log n}^{}\|(\bA-\bP)^{k}\bU\bS\bM_{k}\|_{\ttinf}\\
	&\le \left(\tfrac{d(\log n)^{2}}{n\rho_{n}}\right)\sum_{k=2}^{\log n}\probO\left(\tfrac{d^{1/2}(\log n)^{kc}}{n^{1/2}(n\rho_{n})^{k/2-1/2}}\right)
	+ \left(\tfrac{d(\log n)}{n\rho_{n}}\right)\sum_{k > \log n}\probO(k (n\rho_{n})^{-k/2+1/2})\\
	&= \left(\tfrac{d(\log n)^{2}}{n\rho_{n}}\right)\probO\left(\tfrac{d^{1/2}(\log n)^{2c}}{n^{1/2}(n\rho_{n})^{1/2}}\right)
	+ \left(\tfrac{d(\log n)}{n\rho_{n}}\right)\probO\left((\log n) (n\rho_{n})^{-(\log n)/2}\right)\\
	&= \left(\tfrac{d(\log n)^{2}}{n\rho_{n}}\right)\probO\left(\tfrac{d^{1/2}(\log n)^{2c}}{n^{1/2}(n\rho_{n})^{1/2}}\right)
	+ \left(\tfrac{d(\log n)^{2}}{n\rho_{n}}\right)\probO\left(\tfrac{1}{n^{1/2}(n\rho_{n})^{1/2}}\right).
\end{align*}
Since $(n\rho_{n}) = \omega(d(\log n)^{4c})$ for $c > 1$, we have $(n \rho_{n}) = \omega(d (\log n)^{2})$. It follows that
\begin{equation*}
	\|\bV_{3}\|_{\ttinf}
	= \probO\left(\tfrac{d^{1/2}(\log n)^{2c}}{n^{1/2}(n\rho_{n})^{1/2}}\right).
\end{equation*}
\subsubsection{First and second-order characterization}
In summary, so far we have shown that
\begin{equation}
	\label{eq:scaled_decomposition_summary}
	\hat{\bU}|\hat{\bS}|^{1/2}
	= \bU|\bS|^{1/2}\bW_{\star}
	+ (\bA-\bP)\bU|\bS|^{-1/2}\bW_{\star}\indefI + \bR,
\end{equation}
for some (residual) matrix $\bR \in \R^{n \times d}$ satisfying $\|\bR\|_{\ttinf} = \probO\left(\tfrac{d^{1/2}(\log n)^{2c}}{n^{1/2}(n\rho_{n})^{1/2}}\right)$.

Now let $\bQ_{\bX}$ be such that $\bX = \bU|\bS|^{1/2}\bQ_{\bX}$. Rearranging the terms in Eq.~(\ref{eq:scaled_decomposition_summary}) and multiplying first by $\bW_{\star}^{\top}$ followed by $\bQ_{\bX}$ yields
\begin{align*}
	\hat{\bU}|\hat{\bS}|^{1/2}\bW_{\star}^{\top}\bQ_{\bX}
	- \bU|\bS|^{1/2}\bQ_{\bX}
	&= (\bA-\bP)\bU|\bS|^{-1/2}\indefI\bQ_{\bX} + \bR\bW_{\star}^{\top}\bQ_{\bX}\\
	&= (\bA-\bP)\bU|\bS|^{1/2}\bQ_{\bX}\bQ_{\bX}^{-1}|\bS|^{-1}\indefI\bQ_{\bX} + \bR\bW_{\star}^{\top}\bQ_{\bX}\\
	&= (\bA-\bP)\bX\bQ_{\bX}^{-1}|\bS|^{-1}\indefI\bQ_{\bX} + \bR\bW_{\star}^{\top}\bQ_{\bX}\\
	&= (\bA-\bP)\bX\bQ_{\bX}^{-1}|\bS|^{-1}\indefI\bQ_{\bX}\indefI\indefI + \bR\bW_{\star}^{\top}\bQ_{\bX}\\
	&= (\bA-\bP)\bX\bQ_{\bX}^{-1}|\bS|^{-1}(\bQ_{\bX}^{-1})^{\top}\indefI + \bR\bW_{\star}^{\top}\bQ_{\bX} \\
	&= (\bA-\bP)\bX(\bQ_{\bX}^{\top}|\bS|\bQ_{\bX})^{-1}\indefI + \bR\bW_{\star}^{\top}\bQ_{\bX}.
\end{align*}
Both $\hat{\bX} = \hat{\bU}|\hat{\bS}|^{1/2}$ and $\bQ_{\bX}^{\top}|\bS|\bQ_{\bX} = \bX^{\top}\bX$, yielding the crucial equivalence
\begin{equation}
	\label{eq:GRDPG_ASE_main_eq}
	\hat{\bX}\bW_{\star}^{\top}\bQ_{\bX} - \bX
	= (\bA-\bP)\bX(\bX^{\top}\bX)^{-1}\indefI + \bR\bW_{\star}^{\top}\bQ_{\bX}.
\end{equation}
Theorem~\ref{thrm:GRDPG_ASE_ttinf} holds by observing that
\begin{align*}
	\|\bR\bW_{\star}^{\top}\bQ_{\bX}\|_{\ttinf}
	&= \probO\left(\tfrac{d^{1/2}(\log n)^{2c}}{n^{1/2}(n\rho_{n})^{1/2}}\right)
\end{align*}
and
\begin{align*}
	\|(\bA-\bP)\bX(\bQ_{\bX}^{\top}|\bS|\bQ_{\bX})^{-1}\indefI\|_{\ttinf}
	&= \|(\bA-\bP)\bU|\bS|^{-1/2}\indefI\bQ_{\bX}\|_{\ttinf}
	= \probO\left(\tfrac{d^{1/2}(\log n)^{c}}{n^{1/2}}\right),
\end{align*}
where Lemma~\ref{lem:indefiniteSpectralBoundGRDPG} was implicitly invoked.

For the purpose of establishing Theorem~\ref{thrm:GRDPG_ASE_CLT}, the $i$-th row of Eq.~\ref{eq:GRDPG_ASE_main_eq}, when scaled by $n^{1/2}$, can be written as
\begin{align*}
	n^{1/2}(\bQ_{\bX}^{\top}\bW_{\star}\hat{X}_{i} - X_{i})
	&= n^{1/2}\indefI(\bX^{\top}\bX)^{-1}((\bA-\bP)\bX)_{i} + n^{1/2}\bQ_{\bX}^{\top}\bW_{\star}R_{i},
\end{align*}
where the vector $n^{1/2}\indefI(\bX^{\top}\bX)^{-1}((\bA-\bP)\bX)_{i}$ can be expanded as
\begin{equation*}
	\indefI(n^{-1} \bX^{\top}\bX)^{-1}\left[n^{-1/2}\sum_{j}(\bA_{ij}-\bP_{ij})X_{j} \right] = \indefI(n^{-1} \rho_n^{-1} \bX^{\top}\bX)^{-1}\left[(n \rho_n) ^{-1/2}\sum_{j}(\bA_{ij}-\bP_{ij})\xi_{j}
	\right]
\end{equation*}
by recalling that $X_i = \rho_n^{1/2} \xi_i$. The law of large numbers and the continuous mapping theorem together yield $(n^{-1} \rho_n^{-1} \bX^{\top}\bX)^{-1} \rightarrow \bE(\xi \xi)^{-1} \equiv \bDelta^{-1}$ almost surely.
In addition, the classical multivariate central limit theorem gives the (conditional) convergence in distribution
\begin{equation}
	\Bigl((n\rho_{n})^{-1/2}\sum_{j}(\bA_{ij}-\bP_{ij})\xi_{j} \Bigl \vert \xi_i = x_i \Bigr) \overset{}{\rightarrow} \mathcal{N}_{d}(\bZero, \bGamma_{\rho_{n}}(x_{i})),
\end{equation}
with explicit covariance matrix given by $\bGamma_{\rho_{n}}(x_{i}) = \Ex\left\{(x_{i}^{\top}\indefI \xi)(1-\rho_{n}x_{i}^{\top}\indefI \xi) \xi \xi^{\top}\right\}$. In addition, by combining Lemma~\ref{lem:indefiniteSpectralBoundGRDPG} with our earlier analysis, it follows that the (transformed) residual matrix satisfies
\begin{equation*}
	\|n^{1/2}\bQ_{\bX}^{\top}\bW_{\star}R_{i}\|_{\ttinf}
	\le n^{1/2}\|\bQ_{\bX}\|\|\bW_{\star}\|\|\bR\|_{\ttinf}
	= \probO\left(\tfrac{d^{1/2}(\log n)^{2c}}{(n\rho_{n})^{1/2}}\right)
	\overset{\textnormal{p}}{\rightarrow} 0.
\end{equation*}
The above observations together with an application of Slutsky's theorem yield
\begin{equation}
	\mathbb{P}\left\{n^{1/2}(\bQ_{n}\hat{X}_{i} - X_{i}) \le z \,\, \mid X_{i} = \rho_n^{1/2} x\right\}
	\rightarrow \Phi(z,\bSigma_{\rho_{n}}(x))
\end{equation}
for $\bQ_{n} := \bQ_{\bX_{n}}^{\top}\bW_{\star,n}$ and $\bSigma_{\rho_{n}}(x) = \indefI\bDelta^{-1}\bGamma_{\rho_{n}}(x)\bDelta^{-1}\indefI$. Application of the Cram\'er-Wold device yields Theorem~\ref{thrm:GRDPG_ASE_CLT}, concluding the proof.

\end{document}